\documentclass[conference]{IEEEtran}
\IEEEoverridecommandlockouts

\usepackage[american]{babel}

\usepackage{mathtools}
\usepackage{booktabs} 
\usepackage{algorithm}
\usepackage[switch]{lineno}
\usepackage{hyperref}
\usepackage{breakurl}

\usepackage{microtype}
\usepackage{hyperref}
\usepackage{amsmath}
\usepackage{amsthm}
\usepackage{amssymb}
\usepackage{setspace}
\usepackage{adjustbox}
\usepackage{relsize}
\usepackage{multirow}
\usepackage{multicol}
\usepackage{acronym}
\usepackage{nccmath}
\usepackage{psfrag}
\usepackage{subcaption}
\usepackage{float}
\usepackage{dsfont}
\usepackage{algpseudocode}
\usepackage[font=small]{caption}

\newtheorem{theorem}{Theorem}

\usepackage{natbib}
\usepackage{amsfonts}
\usepackage{graphicx}
\usepackage{textcomp}
\usepackage{xcolor}

\newcommand\V[1]  { \mathbf{#1} }
\newcommand\B[1]  { \boldsymbol{#1} }

\newcommand\set[1] {\mathcal{#1}}

\newcommand\up[1] {\mathrm{#1}}

\acrodef{RFF}[RFF]{random Fourier features}
\acrodef{SS}[SS]{stochastic subgradient}
\acrodef{RRM}[RRM]{robust risk minimization}
\acrodef{ERM}[ERM]{empirical risk minimization}
\acrodef{MRC}[MRC]{minimax risk classifier}
\acrodef{SVM}[SVM]{support vector machine}
\acrodef{LP}[LP]{linear program}
\acrodef{LR}[LR]{logistic regression}
\acrodef{SOA}[SOA]{state-of-the-art}

\def\BibTeX{{\rm B\kern-.05em{\sc i\kern-.025em b}\kern-.08em
    T\kern-.1667em\lower.7ex\hbox{E}\kern-.125emX}}

\begin{document}
\title{Efficient Large-Scale Learning of \\ Minimax Risk Classifiers
}

\author{
\IEEEauthorblockN{
Kartheek Bondugula$^{1}$,
Santiago Mazuelas$^{1,2}$,
Aritz P{\'e}rez$^{1}$
}
\IEEEauthorblockA{
$^{1}$ Basque Center for Applied Mathematics (BCAM), Bilbao, Spain \\
$^{2}$ IKERBASQUE, Basque Foundation for Science \\
\texttt{\{kbondugula, smazuelas, aperez\}@bcamath.org}
}
}


\maketitle

\begin{abstract}
Supervised learning with large-scale data usually leads to complex optimization problems, especially for classification tasks with multiple classes.
Stochastic subgradient methods can enable efficient learning with a large number of samples for classification techniques that minimize the average loss over the training samples. However, recent techniques, such as \acp{MRC}, minimize the maximum expected loss and are not amenable to stochastic subgradient methods. In this paper, we present a learning algorithm based on the combination of constraint and column generation that enables efficient learning of \acp{MRC} with large-scale data for classification tasks with multiple classes. Experiments on multiple benchmark datasets show that the proposed algorithm provides upto a 10x speedup for general large-scale data and around a 100x speedup with a sizeable number of classes.
\end{abstract}

\section{Introduction}

Large-scale data is common in multiple machine learning tasks, such as classification of handwritten digits, or sentiment analysis of text reviews \citep{akata2013good, wang2020survey}. Such data is composed by a large amount of training samples that are often represented as a high-dimensional vector of features \citep{li2010object}. Supervised learning with large-scale data usually leads to complex optimization problems specially for classification tasks with multiple classes \cite{bottou2007tradeoffs, bengio2010label, wang2020survey}. 

\Acf{SS} methods can enable efficient learning for a large number of samples \citep{shalev2007pegasos, bottou2010large, YuaGuo:12, kawaguchi2020ordered}. These methods are amenable for classification techniques based on the \ac{ERM} approach \citep{MehRos:18} that minimizes the average risk over the training samples. Recent techniques, such as \acfp{MRC} \citep{MazZanPer:20, MazRomGrun:23}, are based on the \ac{RRM} approach that minimizes the maximum expected loss over the training samples and can provide performance guarantees in terms of worst-case error probability. Such techniques cannot leverage \ac{SS} since the objective function is given by a maximization instead of an average over samples.

Efficient learning with a large number of features is often achieved by exploiting the parameters' sparsity induced by regularization methods \citep{singer2009efficient, tan2010learning, li2010object, MigKrzLu:20, CelMon:22}. In particular, recent methods have leveraged the parameters' sparsity using constraint generation approaches for linear optimization \citep{DedAntEtal:22, BonMazPer:23}. Specifically, the methods in \cite{DedAntEtal:22} have enabled the efficient learning of \acp{SVM} with L1-penalization for a large number of features and/or instances. In addition, the methods in \mbox{\cite{BonMazPer:23}} have enabled the efficient learning of \acp{MRC} with L1-penalization for a large number of features. However, the methods in \cite{DedAntEtal:22} are proposed for binary classification problems and those in \cite{BonMazPer:23} only provide efficient learning with a small number of samples and classes. Specifically, the complexity per iteration in \cite{BonMazPer:23} increases cubically with the number of samples and exponentially with the number of classes.

In this paper, we present a learning algorithm for \acp{MRC} that enables efficient learning with large-scale data and multiple classes. Specifically, the main contributions of the paper are as follows.
\begin{itemize}
    \item We present an algorithmic framework based on the combination of constraint and column generation methods that enables the efficient learning of \acp{MRC} in general large-scale scenarios.
    \item The presented algorithm utilizes a greedy approach for constraint selection at each iteration that results in a complexity that scale quasi-linearly with the number of classes.
    \item For multiple scenarios of large-scale data, we present theoretical results that describe the convergence of the algorithm proposed.
    \item Using multiple benchmark datasets, we experimentally show that the proposed algorithm enables efficient learning of \acp{MRC} with large-scale data and multiple labels.
\end{itemize}

\textbf{Notations:} For a set $\set{S}$, we denote its cardinality as $|\set{S}|$; bold lowercase and uppercase letters represent vectors and matrices, respectively; for a matrix $\V{F}$, and set of indices $\set{I}$ and $\set{J}$, $\V{F}_{\set{I}, \set{J}}$ denotes the submatrix obtained by the rows and columns corresponding to the indices in $\set{I}$ and $\set{J}$, respectively; for a vector $\V{b}$ and set of indices $\set{I}$, $\V{b}_{\set{I}}$ denotes the subvector obtained by the components corresponding to the indices in $\set{I}$; for a vector $\V{b}$ and an index $i$, $\up{b}_i$ denotes the component at index $i$; $\V{I}$ denotes the identity matrix; $\mathds{1} \{\cdot\}$ denotes the indicator function; $\B{1}$ denotes a vector of ones; for a vector $\V{v}$, $|\V{v}|$ and $(\V{v})_+$, denote its component-wise absolute value and positive part, respectively; $\|\cdot\|_1$ denotes the 1-norm of its argument; $\otimes$ denotes the Kronecker product; $\preceq$ and $\succeq$ denote vector inequalities; $\mathbb{E}_{\up{p}}\{\,\cdot\,\}$ denotes the expectation of its argument with respect to distribution $\up{p}$; and $\V{e}_{i}$ denotes the $i$-th vector in a standard basis. For an integer $a$, we use $[a]$ to denote the set $\{1,2,\ldots,a\}$.
\section{Preliminaries and related work}
In this section, we describe the setting addressed in the paper, and the MRC methods that minimize the worst-case error probability.

\vspace{-0.1cm}
\subsection{Problem formulation}
\vspace{-0.1cm}
Supervised classification uses sample-label pairs to determine classification rules that assign labels to samples. We denote by $\mathcal{X}$ and $\mathcal{Y}$ the sets of samples and labels, respectively, with $\mathcal{Y}$ represented by $\{1, 2, \ldots, |\mathcal{Y}|\}$. We denote by $\Delta(\mathcal{X}\times\mathcal{Y})$ the set of probability distributions on $\mathcal{X} \times \mathcal{Y}$ and by $\ell(\up{h}, \up{p})$ the expected 0-1 loss of  the classification rule $\up{h}$ with respect to distribution $\up{p} \in \Delta(\mathcal{X}\times\mathcal{Y})$, i.e., $\ell(\up{h},\up{p})=\mathbb{E}_{\up{p}} \mathds{1}\{\up{h}(x)\neq y\}$.

Sample-label pairs can be embedded into real-valued vectors using a feature map $\Phi: \set{X}\times \set{Y} \rightarrow \mathbb{R}^m$.  The most common approach to construct such a mapping involves combining multiple features over the samples with one-hot encodings of the labels as follows (see e.g., \cite{MehRos:18})
\begin{align}
\label{eq:feature}
 \Phi(x, y)  =  \big{[}\mathds{1}\{y = 1\} {\Psi}(x)^{\text{T}},  \ldots, \mathds{1}\{y = |\mathcal{Y}|\} {\Psi}(x)^{\text{T}}\big{]}^{\text{T}}\nonumber
\end{align}
 where the map $\Psi : \mathcal{X} \rightarrow \mathbb{R}^d$ represents samples as real vectors of size $d$, for example, using random Fourier features \citep{RahRec:08}.
 

In this paper we consider scenarios in which the number of samples $n$, number of classes $|\set{Y}|$, and features' dimensionality $m = d|\set{Y}|$ are large, and propose efficient learning methods for the \acp{MRC}, which are briefly described in the following.

\subsection{Minimax Risk Classifiers}
\vspace{-0.1cm}
Minimax risk classifiers (MRCs) are classification rules that minimize the worst-case error probability over distributions in an uncertainty set \citep{MazZanPer:20, MazSheYua:22, MazRomGrun:23}. Specifically, such rules are solutions to the minimax risk problem defined as
\begin{equation}
\label{eq:minmaxrisk}
\up{R}^{*} = \underset{\up{h}}{\min} \, \underset{\up{p} \in \mathcal{U}}{\max} \; \ell(\up{h}, \up{p})
\end{equation}
where $\up{R}^{*}$ is the worst-case error probability and $\set{U}$ is an uncertainty set of distributions determined by expectation estimates as
\begin{equation}
\label{eq:us}
\mathcal{U} = \{\up{p} \in \Delta (\mathcal{X} \times \mathcal{Y}) : \left|\mathbb{E}_{\up{p}}\{\Phi(x, y)\}  - \B{\tau} \right| \preceq \B{\lambda}\}.
 \end{equation}
The mean vector $\B{\tau}$ denotes expectation estimates corresponding with the feature mapping $\Phi: \set{X}\times \set{Y} \rightarrow \mathbb{R}^m$, and the confidence vector $\B{\lambda} \succeq \V{0}$ accounts for inaccuracies in the estimate. The mean and confidence vectors can be obtained from the training samples $\{(x_i,y_i)\}_{i=1}^n$ as $\B{\tau}=\frac{1}{n}\sum_{i=1}^{n}\Phi(x_i,y_i)$, \mbox{$\B{\lambda}= \lambda_0\V{s}$},
where $\V{s}$ denotes the vector formed by the component-wise sample standard deviations of $\{\Phi(x_i,y_i)\}_{i=1}^n$ and $\lambda_0$ is a regularization parameter. 

As described in  \cite{MazZanPer:20, MazSheYua:22, MazRomGrun:23} using the expected 0-1 loss, the MRC rule solution of~\eqref{eq:minmaxrisk} is given by a linear combination of the feature mapping. Specifically, the minimax rule assigns labels as $\arg \max_{y \in \set{Y}} {\Phi}(x,y)^{\text{T}} \B{\mu}^*$. The vector $\B{\mu}^*$ that determines the \mbox{0-1} \ac{MRC} rule corresponding to \eqref{eq:minmaxrisk} is obtained by solving the convex optimization problem \citep{MazZanPer:20,MazRomGrun:23}
\begin{equation}
\label{eq:mrc}
\underset{\B{\mu} \in \mathbb{R}^m}{\min} \; 1 - \B{\tau}^{\text{T}} \B{\mu} + \varphi(\B{\mu}) + \B{\lambda}^{\text{T}} | \B{\mu} |, 
\end{equation}
with $\varphi(\B{\mu}) = {\underset{x \in \mathcal{X}, \set{C} \subseteq \mathcal{Y}}{\max}} (\sum_{y \in \set{C}}\Phi(x, y)^{\text{T}}\B{\mu} - 1) / |\set{C}|$. In addition, the minimum value of \eqref{eq:mrc} equals the minimax risk value of \eqref{eq:minmaxrisk}. The minimax risk value provides an upper bound on the expected classification error when the underlying distribution of the training samples is included in the uncertainty set.

The convex optimization problem \eqref{eq:mrc} for \acp{MRC} includes an L1-penalty $\B{\lambda}^{\up{T}}|\B{\mu}|$, which induces sparsity in the coefficients $\B{\mu}^{*}$ associated with the feature mapping. This sparsity indicates that only a subset of features is necessary to achieve the optimal worst-case error probability. Consequently, efficient learning can be realized by focusing on the most relevant features when solving the \ac{MRC} optimization problem.

\vspace{-0.1cm}
\subsection{LP formulation of MRCs}
\vspace{-0.1cm}
\label{sec:mrc_lp}
The convex optimization problem \eqref{eq:mrc} of \mbox{0-1} \acp{MRC} can be formulated as the following \ac{LP} \citep{BonMazPer:23}
\begin{align}
\label{eq:mrc_linear_primal}
\arraycolsep=1pt\def\arraystretch{1}
\begin{array}{ccc}
\mathcal{P}: & \underset{\B{\mu}_1, \B{\mu}_2, \nu}{\min} & - (\B{\tau} - \B{\lambda})^{\up{T}}\B{\mu}_1 + (\B{\tau} + \B{\lambda})^{\up{T}}\B{\mu}_2 + \nu \\ 
 & \text{s.t.} & \B{\up{F}}(\B{\mu}_1 - \B{\mu}_2) - \nu\B{1} \preceq \B{\up{b}}, \ \B{\mu}_1, \B{\mu}_2 \succeq 0.
\end{array}
\end{align}
The matrix $\V{F}$ and vector $\V{b}$ that define the constraints can be represented as 
\begin{align}
\label{eq:F_submat}
\V{F} = \begin{bmatrix}
\V{F}^{x_1}\\
\V{F}^{x_2}\\
\vdots\\
\V{F}^{x_n}
\end{bmatrix}
, \ \V{b} =\begin{bmatrix}
\hat{\V{b}}\\
\hat{\V{b}}\\
\vdots\\
\hat{\V{b}}
\end{bmatrix}
\end{align}
such that a pair of submatrix $\V{F}^{x_i}$ and subvector $\hat{\V{b}}$ defines all the constraints corresponding to each training instance $x_i$, for $i=1,2,\ldots,n$.
Each row of the submatrix $\V{F}^{x_i}$ and component of subvector $\hat{\V{b}}$ corresponds to a non-empty subset $\set{C}\subseteq \set{Y}$, and is defined as $\sum_{y \in \mathcal{C}}\Phi\left(x_i,y\right)^{\up{T}} / |\mathcal{C}|$ and $1 / \mathcal{|C|}-1$, respectively, for a given $x_i \in \set{X}$. The number of variables $q=2m+1$ in \eqref{eq:mrc_linear_primal} is given by the number of features $m$ while the number of constraints $p = n(2^{|\set{Y}|} - 1)$ in \eqref{eq:mrc_linear_primal} is given by the number of training samples $n$ and the non-empty subsets of $\set{Y}$. Optimization problem \eqref{eq:mrc} can be solved by addressing the \ac{LP} in \eqref{eq:mrc_linear_primal} at learning because the optimal values of \eqref{eq:mrc} and \eqref{eq:mrc_linear_primal} coincide, and a solution of \eqref{eq:mrc} $\B{\mu}^*$ can be obtained as $\B{\mu}^* = \B{\mu}_1^* - \B{\mu}_2^*$ for $\B{\mu}_1^*$ and $\B{\mu}_2^*$ solution of \eqref{eq:mrc_linear_primal}.




The lagrange dual of \eqref{eq:mrc_linear_primal} is given by
\begin{align}
\label{eq:mrc_linear_dual}
\arraycolsep=1pt\def\arraystretch{1}
\begin{array}{ccc}
\mathcal{D}: & \underset{\B{\alpha}}{\max} & -\B{\up{b}}^{\up{T}}\B{\alpha} \\
& \text{s.t.} & \B{\tau} - \B{\lambda} \ \preceq \ \B{\up{F}}^{\up{T}}\B{\alpha} \ \preceq  \ \B{\tau} + \B{\lambda} \\
& & \B{1}^{\up{T}}\B{\alpha} = 1, \ \B{\alpha} \succeq 0.
\end{array}
\end{align}
\vspace{-0.1cm}
The number of variables in the dual \ac{LP} in \eqref{eq:mrc_linear_dual} is equal to the number of constraints in the primal, that is, $p$, while the number of constraints in \eqref{eq:mrc_linear_dual} is equal to the number of variables in the primal, that is, $q$. 

The complexity of \ac{MRC} learning given by such \ac{LP} formulation is not affordable when either the number of samples $n$, the number of features $m$, or the number of classes $|\set{Y}|$ are large. In particular, the number of variables and constraints in \eqref{eq:mrc_linear_primal} grow linearly with $m$ and $n$, respectively, while the number of constraints in \eqref{eq:mrc_linear_primal} grows exponentially with $|\set{Y}|$. In the following, we describe the general constraint generation methodology that we use in the remaining of the paper to provide an efficient learning algorithm for large-scale \acp{MRC}.


\subsection{Constraint Generation for Linear Programming}
\vspace{-0.1cm}
\label{sec:constraint_generation}
Constraint generation methods can enable to efficiently solve large-scale \acp{LP} when the number of constraints is large \citep{DimTsi97, DesJac}. The basic idea is to start with a candidate set of constraints and iteratively add new constraints until the resulting solution is feasible for the whole problem. Such procedure is effective when the number of variables is significantly smaller than the number of constraints since only a small set of constraints will be active in the solution. Moreover, this procedure can also be used for large number of variables by considering the corresponding dual \ac{LP}. Since generating (adding) constraints in the dual is equivalent to adding variables in the primal, such method is also known as column generation. In the following, we use constraint and column generation tools to provide a learning algorithm for the \ac{LP} formulation of \acp{MRC} that is efficient for large-scale data and multi-class classification.

\section{Efficient Large-Scale Learning of MRCs}
In this section, we propose a learning algorithm based on the combination of constraint and column generation for the \ac{LP} formulation of \mbox{0-1} \acp{MRC}. We present implementation details along with theoretical guarantees for multiple scenarios such as large number of samples but few features, and large number of samples with many features. In addition, we describe efficient techniques for learning with multiple classes.

\subsection{General Algorithm for Efficient Learning}
\label{sec:general}
In the following, we present an iterative algorithm based on the combination of constraint and column generation that efficiently solves the 0-1 \acp{MRC} \ac{LP} \eqref{eq:mrc_linear_primal} (see Algorithm~\ref{alg:efficient_mrc_for_large_n_m} for the pseudocode). Efficient learning for \acp{MRC} can be achieved since only a small subset of variables and constraints are active at the solution of the \ac{LP} formulation for multiple settings of large-scale data, as detailed in \mbox{Sections~\ref{subsec:large_n}} and~\ref{subsec:large_n_m}.

The proposed Algorithm~\ref{alg:efficient_mrc_for_large_n_m} obtains the optimal solution $\B{\mu}^*$ by iteratively solving a sequence of subproblems.
\setlength{\textfloatsep}{1pt}
\vskip -0.1cm
\begin{algorithm}[ht]
\captionsetup{labelfont={bf}, format=hang}
\caption{Efficient learning of MRCs for large-scale data}
\label{alg:efficient_mrc_for_large_n_m}
\begin{tabular}{ll}
\textbf{Input:} & \hspace{-0.4cm} $\V{F}, \V{b}, \B{\tau}, \text{ and } \B{\lambda}$, \\
& \hspace{-0.4cm} initial subset of features $\set{J}$ and constraints $\set{I}$, \\
& \hspace{-0.4cm} initial solution $\B{\mu}_1^1$, $\B{\mu}_2^1$, $\nu^1$ of $\mathcal{P}_{\set{I}, \set{J}}$, \\
& \hspace{-0.4cm} initial solution $\B{\alpha}^1$ of $\mathcal{D}_{\set{I}, \set{J}}$, \\
& \hspace{-0.4cm} primal constraints' violation threshold $\epsilon_1$, \\
& \hspace{-0.4cm} dual constraints' violation threshold $\epsilon_2$, \\
& \hspace{-0.4cm} constraint limit per iteration $n_\text{max}$, \\
& \hspace{-0.4cm} feature limit per iteration $m_\text{max}$. \\
& \hspace{-0.4cm} maximum number of iterations $k_\text{max}$. \\
\textbf{Output:} & \hspace{-0.3cm}optimal solution $\B{\mu}^{*} \in \mathbb{R}^{m}$, \\
& \hspace{-0.3cm}worst-case error probability $\up{R}^{*}$. \\
\end{tabular}
\begin{algorithmic}[1]
\setstretch{1.2}
\State $k \leftarrow 1$
\Repeat
\State $\hat{\set{J}} \leftarrow \set{J}, \ \hat{\set{I}} \leftarrow \set{I}$
\State $\set{J} \leftarrow \text{FEAT}(\hat{\set{I}}, \ \hat{\set{J}}, \ \B{\alpha}^k, \ \epsilon_2, \ m_\text{max})$
\State $\set{I} \leftarrow \text{CONSTR}(\hat{\set{I}}, \ \hat{\set{J}}, \ \B{\mu}_1^{k},\ \B{\mu}_2^{k},\
 \nu^k, \ \epsilon_1, \ n_\text{max})$
\State LPSOLVE$\big(\set{I}, \set{J}\big)$ \label{alg:efficient_mrc:line:11}
\Statex \hspace{0.4cm} $\B{\mu}_1^{k+1}, \ \B{\mu}_2^{k+1}, \ \nu^{k+1} \leftarrow$ Solution of primal $\mathcal{P}_{\set{I}, \set{J}}$
\Statex \hspace{0.4cm} $\B{\alpha}^{k+1} \leftarrow \text{Solution of dual }\mathcal{D}_{\set{I}, \set{J}}$
\Statex \hspace{0.4cm} $\up{R}^{k+1} \leftarrow$ Optimal value
\State $k \leftarrow k + 1$
\Until{$\set{I} \setminus \hat{\set{I}} = \emptyset$ and $\set{J} \setminus \hat{\set{J}} = \emptyset$ and $k \leq k_\text{max}$}
\State $\B{\mu}^{*} = [0,0,\ldots,0] \in \mathbb{R}^m$
\State $\B{\mu}^{*}_{\set{J}} \leftarrow \B{\mu}_1^k - \B{\mu}_2^k$, \ $\up{R}^{*} \leftarrow \up{R}^{k}$
\end{algorithmic}
\end{algorithm}
\vskip -0.2cm
These subproblems correspond to the \ac{LP} in \eqref{eq:mrc_linear_primal} defined over a subset of constraints and features. Specifically, the subproblem corresponding to the subset of constraints $\set{I} \subseteq \{1, 2, \ldots, p\}$ and features $\set{J} \subseteq \{1, 2, \ldots, m\}$ is
defined as
\begin{align}
\label{eq:mrc_linear_subprob_primal}
\arraycolsep=0.5pt\def\arraystretch{1}
\begin{array}{ccc}\mathcal{P}_{\set{I}, \set{J}}: & \underset{\B{\mu}_1, \B{\mu}_2, \nu}{\min} & - (\B{\tau} - \B{\lambda})_{\mathcal{J}}^{\up{T}}\B{\mu}_1 + (\B{\tau} + \B{\lambda})^{\up{T}}_{\mathcal{J}}\B{\mu}_2 + \nu \\
 & \text{s.t.} & \B{\up{F}}_{\set{I}, \set{J}}(\B{\mu}_1 - \B{\mu}_2) - \nu\B{1} \preceq \V{b}_\set{I}, \ \B{\mu}_1, \B{\mu}_2 \succeq 0.
 \end{array}
\end{align}
In addition, the dual of \eqref{eq:mrc_linear_subprob_primal} is
\begin{equation}
\label{eq:mrc_linear_subprob_dual}
\arraycolsep=0.5pt\def\arraystretch{1}
\begin{array}{ccc}
\hspace{-0.1cm}\mathcal{D_{\set{I}, \set{J}}}: & \underset{\B{\alpha}}{\max} & -\V{b}_{\set{I}}^{\up{T}}\B{\alpha} \\
& \text{s.t.} & (\B{\tau} - \B{\lambda})_{\set{J}} \ \preceq \ ({\B{\up{F}}}_ {\set{I}, \set{J}})^{\up{T}}\B{\alpha} \ \preceq  \ (\B{\tau} + \B{\lambda})_{\set{J}} \\
& & \B{1}^{\up{T}}\B{\alpha} = 1, \B{\alpha} \succeq 0.
\end{array}
\end{equation}

Each iteration $k$ of Algorithm~\ref{alg:efficient_mrc_for_large_n_m} solves subproblems \eqref{eq:mrc_linear_subprob_primal} and \eqref{eq:mrc_linear_subprob_dual}, and obtains primal solution $\B{\mu}^k$, $\nu^k$ and the dual solution $\B{\alpha}^k$ along with the worst-case error probability $\up{R}^k$ given by the optimal value. The primal solution is used by the function CONSTR to obtain the subsequent primal constraints based on the constraints' violation. The dual solution is used by the function FEAT to obtain the subsequent set of features based on the dual constraints' violation.

The function FEAT obtains the subsequent set of features $\set{J}$ by adding and/or removing features based on the violations in the corresponding dual constraints. In particular, each feature $j \in [m]$ corresponds to two variables in the primal ${\mu_1}^k_j$ and ${\mu_2}^k_j$ that correspond to the two constraints $(\V{F}_{\hat{\set{I}}, j})^\up{T}\B{\alpha}^k \geq \tau_j - \lambda_j$ and $(\V{F}_{\hat{\set{I}}, j})^\up{T}\B{\alpha}^k \leq \tau_j + \lambda_j$ in the dual. Therefore, a new feature $j \in [m] \setminus \hat{\set{J}}$ is added to the set $\set{J}$ if one of the corresponding dual constraints is violated by at least $\epsilon_2 \geq 0$, that is, $|(\V{F}_{\hat{\set{I}}, j})^\up{T}\B{\alpha}^k - \tau_j| - \lambda_j \geq \epsilon_2$. In addition, an existing feature $j \in \set{J}$ is removed if the corresponding dual constraints are overly-satisfied, that is, $|(\V{F}_{\hat{\set{I}}, j})^\up{T}\B{\alpha}^k - \tau_j| - \lambda_j < 0$. Similarly, the function CONSTR obtains the subsequent set of primal constraints $\set{I}$ by adding and/or removing constraints based on the violations. In particular, a new constraint $i \in [p] \setminus \hat{\set{I}}$ is added to the set $\set{I}$ if it is violated by at least $\epsilon_1 \geq 0$, that is, $\V{F}_{i,\hat{\set{J}}}(\B{\mu}_1^k - \B{\mu}_2^k) - \nu^k - \up{b}_i \geq \epsilon_1$. In addition, an existing constraint $i \in \set{I}$ is removed if it is overly-satisfied, that is, ${\V{F}}_{i, \hat{\set{J}}}(\B{\mu}_1^k - \B{\mu}_2^k) - \nu^k - \up{b}_i < 0$.

The function FEAT adds $m_\text{max}$ features corresponding to the $m_\text{max}$ most violated constraints in the dual. This greedy selection process requires evaluating all the constraints in the dual so that its complexity is $O(q)$. On the other hand, the function CONSTR adds $n_\text{max}$ violated constraints of the primal using the greedy algorithm described in the following to achieve a significantly lower complexity than $O(p)$.

\subsubsection{Efficient evaluation of primal constraints violation}
\label{subsec:constr_eval}
The number of constraints $p$ in the primal \ac{MRC} \ac{LP} \eqref{eq:mrc_linear_primal} is given by $n(2^{|\set{Y}|}-1)$. Evaluating all the constraints to find the maximum violation has an exponential complexity in terms of the number of classes $|\set{Y}|$ that is not affordable in multi-class settings. In the following, we present Algorithm~\ref{alg:greedy_constraint_check} that computes the constraint with maximum violation in a significantly lower complexity.

\vspace{-0.2cm}
\begin{algorithm}[ht]
\captionsetup{labelfont={bf}, format=hang}
\caption{Efficient evaluation of constraints violation}
\label{alg:greedy_constraint_check}
\begin{tabular}{ll}
\textbf{Input:} & \hspace{-0.4cm} instance $x \in \set{X}$, \\
& \hspace{-0.4cm} feature mapping $\Phi$, \\
& \hspace{-0.4cm} primal solution $\B{\mu}$ \\
\textbf{Output:} & \hspace{-0.3cm}index $i^*$ for the subset \\ 
& \hspace{-0.4cm} corresponding to maximum in $\varphi(\B{\mu})$ \\
\end{tabular}
\begin{algorithmic}[1]
\setstretch{1.2}
\State $\B{\up{v}} = [\Phi(x,1), \Phi(x,2), \ldots, \Phi(x,|\set{Y}|)]^{\up{T}}\B{\mu}$
\State $i_1, i_2, \ldots, i_{|\set{Y}|}$ = ARGSORT($\B{\up{v}}$)
\State $\V{c} = [0,0,\ldots,0] \in \mathbb{R}^{|\set{Y}|}$
\State $\psi = \up{v}_{i_1} - 1, \up{c}_{i_1} = 1$
\For{$k = 2,3,\ldots,|\set{Y}|$}
\State $\hat{\psi} = ((k-1)\psi + \up{v}_{i_{k}}) / k$
\If{$\hat{\psi} \geq \psi$}
\State $\psi = \hat{\psi}$, $\up{c}_{i_k} = 1$
\EndIf
\EndFor
\State $i^* = \sum_{j=1}^{|\set{Y}|}\up{c}_j2^{j-1}$ 
\end{algorithmic}
\end{algorithm}
\vspace{-0.2cm}
The \ac{MRC} \ac{LP} \eqref{eq:mrc_linear_primal} has exponential number of constraints corresponding to each instance $x_i \in \set{X}$. The exponential number of constraints correspond to different non-empty subsets $\set{C} \subseteq \set{Y}$ as defined by the rows and components of submatrix $\V{F}^{x_i}$ and subvector $\hat{\V{b}}$ in \eqref{eq:F_submat} for $x_i \in \set{X}$. The Algorithm~\ref{alg:greedy_constraint_check} returns the index $i^* \in \{1,2\ldots,2^{|\set{Y}|}-1\}$ corresponding with the subset achieving the maximum constraint.
Specifically, a subset achieving the maximum constraint corresponds to the subset achieving the maximum value 
\begin{equation}
   \psi = \underset{\set{C} \subseteq \set{Y}}{\max} \frac{\sum_{y\in\set{C}} \Phi(x_i,y)^{\up{T}}\B{\mu} - 1}{|\set{C}|}
\end{equation}
for an $x_i \in \set{X}$ and given solution $\B{\mu}$.
The Algorithm~\ref{alg:greedy_constraint_check} computes the maximum value $\psi$ over all subsets in a greedy fashion similar to the approach presented in \cite{FathLiuAsi:16}.
In particular, the algorithm computes $\Phi(x_i, y)^{\up{T}}\B{\mu}$ corresponding to each $y \in \set{Y}$ and sorts them in decreasing order. Then, the algorithm starts with empty set $\set{C}=\emptyset$ and iteratively adds a label to the subset $\set{C}$ in the sorted order until adding a label does not increase the value of $\hat{\psi}$. Such greedy approach obtains the maximum value $\psi$ over all the subsets since the maximum value for a fixed subset length is obtained by the set with labels corresponding with maximum values of $\Phi(x_i, y)^{\up{T}}\B{\mu}$. Therefore, the computational complexity of Algorithm~\ref{alg:efficient_mrc_for_large_n_m} is $O(|\set{Y}|\log|\set{Y}|)$ due to the sorting in line 2.



In the following, we present the implementation details along with the theoretical properties of the proposed learning algorithm for multiple scenarios.

\vspace{-0.1cm}
\subsection{Learning with a large number of samples}
\label{subsec:large_n}
A large number of samples together with a reduced number of features leads to an \ac{MRC} primal \ac{LP} with a large number of constraints and few variables. In such scenario, usually a small set of constraints will be active in the solution \cite{DimTsi97}. Algorithm~\ref{alg:efficient_mrc_for_large_n_m} enables efficient learning by iteratively selecting the relevant subset of constraints using the CONSTR function. In particular, such a function adds and removes constraints (as defined in Section~\ref{sec:general}) along the iterations while considering all the features, that is, the function FEAT is disabled in this scenario. 


The following theorem shows that the proposed algorithm for a large number of samples provides an increasing sequence of worst-case error probabilities and converges to the worst-case error probability given by the \ac{MRC} \ac{LP} corresponding with all the constraints.

\begin{theorem}
\label{th:convergence_n_greater_than_m}
    Let $\up{R}^*$ be the worst-case error probability obtained by solving \eqref{eq:mrc_linear_primal} using all the constraints and features. If $\up{R}^{k}, k=1,2\ldots,$ is the sequence of optimal values obtained by adding and removing constraints along the iterations of the proposed algorithm. Then, we have 
    \begin{align}
        \label{ineq:increasing_worst_case_risk}
        \up{R}^{k} \leq \up{R}^{k+1}.
    \end{align}
    Moreover, if $\hat{\epsilon}_1$ is the largest violation in the constraints of the primal at iteration $k_0$, then at any iteration $k \geq k_0$,
    \begin{align}
        \label{ineq:convergence_worst_case_risk_large_n}
        \up{R}^{*} - \hat{\epsilon}_1 \leq \up{R}^k \leq \up{R}^{*}.
    \end{align}
\vspace{-0.6cm}    .
\end{theorem}

\begin{proof}
See Appendix \ref{subsec:proof_th1}.
\end{proof}

\vspace{-0.2cm}
Inequality \eqref{ineq:increasing_worst_case_risk} shows that the algorithm obtains an increasing sequence of worst-case error probabilities along the iterations. On the other hand, the inequality \eqref{ineq:convergence_worst_case_risk_large_n} shows that the algorithm finds the \ac{MRC} corresponding to all the constraints when there is no violation in the primal constraints, that is, $\hat{\epsilon}_1 = 0$. Such a case occurs after a finite number of iterations due to the properties of constraint generation methods \cite{DimTsi97}. In other cases, it finds an approximate solution with accuracy that depends directly on the largest violation $\hat{\epsilon}_1$. Moreover, if the algorithm terminates with $k<k_\text{max}$, then the hyper-parameter $\epsilon_1$ is the largest violation which can provide a trade-off between the accuracy and the complexity of the algorithm (larger values for $\epsilon_1$ decrease the number of iterations but lead to approximate solutions). In the following, we present the implementation details for the scenario with large number of samples and features along with the corresponding theoretical analysis.
% ...


\vspace{-0.2cm}
\subsection{Learning with a large number of samples and features}
\vspace{-0.1cm}
\label{subsec:large_n_m}
A large number of samples and features corresponds to large number of constraints and variables in the \ac{MRC} primal \ac{LP}. Efficient learning can be achieved in such scenarios leveraging the sparsity in the \acp{MRC} solution due to the L1-penalization. In particular, the solution sparsity implies that \acp{MRC} learning can be carried out using a small subset of features, that is, a small subset of variables in the \ac{MRC} \ac{LP}. Moreover, a small subset of variables implies that only a subset of constraints will be active in the solution. Algorithm~\ref{alg:efficient_mrc_for_large_n_m} enables efficient learning by iteratively selecting the relevant variables and constraints. In particular, Algorithm~\ref{alg:efficient_mrc_for_large_n_m} generates the relevant subset of features and constraints using the functions FEAT and CONSTR that add features and constraints (as defined in Section~\ref{sec:general}) along the iterations. In this scenario, functions CONSTR and FEAT do not remove constraints or features to ensure a proper convergence of the algorithm.

The following theorem bounds the worst-case error probability obtained by the proposed algorithm at any iteration. In addition, such result shows that the worst-case error probability converges to the worst-case error probability of the \acp{MRC} \ac{LP} corresponding with all the constraints and features.
\begin{theorem}
\label{th:convergence_of_cg_cp}
    Let $\B{\mu}^{*}$ and $\up{R}^{*}$ be the \ac{MRC} coefficient and the worst-case error probability obtained by solving \eqref{eq:mrc_linear_primal} using all the constraints and features. If $\up{R}^{k}$ is the worst-case error probability obtained at iteration $k$ of Algorithm~\ref{alg:efficient_mrc_for_large_n_m}. Then,
    \begin{align}
        \label{ineq:convergence_worst_case_risk_large_n_m}
        \up{R}^{*} - \hat{\epsilon}_1 \leq \up{R}^{k} \leq \up{R}^{*} + \hat{\epsilon}_2{\|\B{\mu}^*\|}_1
    \end{align}
    for the largest violations $\hat{\epsilon}_1$ and $\hat{\epsilon}_2$ in the constraints of the primal and dual at iteration $k$.
\end{theorem}
\begin{proof}
See Appendix \ref{subsec:proof_th2}.
\end{proof}

\vspace{-0.2cm}
Inequality \eqref{ineq:convergence_worst_case_risk_large_n_m} shows that the algorithm finds an \ac{MRC} corresponding with all the features and constraints when there is no violation, that is, $\hat{\epsilon}_1=0$ and $\hat{\epsilon}_2=0$. Such a case occurs after a finite number of iterations as the algorithm adds multiple constraints and features in each iteration. 
In case $\hat{\epsilon}_1>0$ and $\hat{\epsilon}_2=0$, the algorithm can under-estimate the worst-case error while in case of $\hat{\epsilon}_1=0$ and $\hat{\epsilon}_2>0$, the algorithm can over-estimate the worst-case error.
In other cases, it obtains an approximate solution that depends on the largest violations $\hat{\epsilon}_1$ and $\hat{\epsilon}_2$. Furthermore,  if the algorithm terminates with $k<k_\text{max}$, then the hyper-parameters $\epsilon_1$ and $\epsilon_2$ are the largest violations and can provide a trade-off between training time and the optimality of the solution.

\vspace{-0.1cm}
\subsection{Computational complexity}
\vspace{-0.1cm}
The computational complexity of Algorithm~\ref{alg:efficient_mrc_for_large_n_m} is determined by two main factors: the number of iterations and the cost per iteration. These, in turn, depend on the maximum number of features and constraints selected in each iteration, denoted by $m_{\up{max}}$ and $n_{\up{max}}$, as well as the constraint violation thresholds $\epsilon_1$ and $\epsilon_2$. Reducing $n_{\up{max}}$ or $m_{\up{max}}$ lowers the per-iteration cost but typically increases the number of iterations required. Conversely, increasing $\epsilon_1$ or $\epsilon_2$ can reduce both the number of iterations and the per-iteration complexity at the cost of obtaining approximate solutions. In addition to the hyper-parameters $n_{\up{max}}$, $m_{\up{max}}$, $\epsilon_1$, and $\epsilon_2$, the per-iteration complexity also depends on the total number of features $m$, the number of samples $n$, and the number of classes $|\set{Y}|$. Specifically, the functions CONSTR and FEAT have time complexities of $O(n|\set{Y}|\log |\set{Y}|)$ and $O(m\log m)$, respectively. The complexity of CONSTR function arises due to a scan over all samples using Algorithm~\ref{alg:greedy_constraint_check}, while the complexity of FEAT function arises from sorting operations required to identify features associated with the most violated dual constraints (see Section~\ref{sec:general} for details).

The number of iterations can be further minimized by carefully selecting the initial subsets of constraints and features. Additionally, choosing an appropriate initial set of constraints is essential to ensure that the feasible region of the \ac{MRC} \ac{LP} is bounded from the very first iteration. A common strategy in constraint generation methods involves using the solution obtained from running a few iterations of a first-order optimization method as a basis for constructing these initial subsets \citep{DedAntEtal:22}. However, the computational cost of first order methods for \acp{MRC} can be large with many samples. We propose to use a simple approach based on clustering to obtain a reasonable choice for the subset of constraints. On the other hand, the corresponding subset of features can be directly obtained using the method in \cite{BonMazPer:23}. 
Specifically, a subset of constraints are obtained as that corresponding to $|\set{Y}|$ samples representing the centers of clusters for different classes. In addition, the constraints corresponding to such samples ensure a bounded feasible region (see Section~\ref{subsec:initialization} for the details).

The complexity per iteration can further improved in case of large number of samples with few features by leveraging the warm-start in \acp{LP}.
In such cases, the warm-start for iteration $k$ is directly derived from the dual solution $\B{\alpha}^{k-1}$ of the previous iteration $k-1$. The warm-start serves as a basic feasible solution for the dual \ac{LP} at iteration $k$, as it is constructed by discarding the primal constraints with positive slack or the dual variables with zero coefficients from the solution at iteration $k-1$.

\subsection{Initialization}
\vspace{-0.1cm}
\label{subsec:initialization}
An adequate choice for the initial subset of constraints and features can reduce the number of iterations for termination of Algorithm~\ref{alg:efficient_mrc_for_large_n_m}. Moreover, an adequate subset of constraints is required to ensure a bounded feasible region for the initial \ac{LP} at the first iteration of Algorithm~\ref{alg:efficient_mrc_for_large_n_m}. In the following, we present an efficient approach based on clustering to obtain an appropriate initial subset of constraints $\set{I}$ and avoid the unbounded feasible region for the initial \ac{LP}. Note that the corresponding initial subset of features $\set{J}$ can be directly obtained by a few iterations of the learning algorithm presented in \cite{BonMazPer:23} on the \ac{LP} defined by $\set{I}$. 

Consider the following $|\set{Y}|$ samples corresponding to the center of the clusters for each class.
\vspace{-0.1cm}
\begin{align}
\label{def:x_art}
    \hat{\set{X}} = \{\hat{\B{x}}_i : \hat{\B{x}}_i = \B{\tau}^i / \up{p}^{y}_{i} \ \forall \ i = 1,2,\ldots,|\set{Y}|\}
\end{align}
where $\B{\tau}^i \in \mathbb{R}^{d}$ denotes the components of $\B{\tau}$ corresponding to class $i$ and $\B{\up{p}}^{y} \in \mathbb{R}^{|\set{Y}|}$ is a vector of class proportions, that is, $\up{p}^{y}_i = \sum_{j=0}^n \mathds{1}\{y_j = i\} / n$ for $n$ training samples. The \ac{MRC} optimization problem \eqref{eq:mrc} solved using samples in $\hat{\set{X}}$, and the estimates $\B{\tau}$ and $\B{\lambda}$ leads to a bounded \ac{LP} since the uncertainty set \eqref{eq:us} is non-empty. Specifically, the $|\set{Y}|$ samples ensure non-emptiness as a distribution with weights $\B{\up{p}}^{y}$ is inside the set since $\sum_{i=0}^{|\set{Y}|}{\up{p}^{y}_i}{\Phi}(\hat{\B{x}}_i, y_i) = \B{\tau}$ for $\hat{\B{x}}_i \in \hat{\set{X}}$. Therefore, the initial \acp{MRC} \ac{LP} \eqref{eq:mrc_linear_primal} given by the set of constraints corresponding to the samples in $\hat{\set{X}}$ is bounded.

The initial \ac{MRC} \ac{LP} using all the $|\set{Y}|(2^{|\set{Y}|} - 1)$ constraints defined over the $|\set{Y}|$ samples in \eqref{def:x_art} ensures a bounded feasible region. Initializing Algorithm~\ref{alg:efficient_mrc_for_large_n_m} using such constraints is a viable option upto a few labels. However, the initialization can be computational inefficient with a large number of classes as the number of constraints in the initial \ac{LP} become significantly large. Therefore, we propose a variation of the above initialization that can be useful for a large number of classes. 
Specifically, we use a reduced subset from the constraints obtained from $\hat{\set{X}}$ as initialization over an \acp{MRC} \ac{LP} with an additional constraint enforcing the objective to be positive. 
Notice that the objective value of the \ac{MRC} \ac{LP} has to be a positive value since it is the worst-case error probability. Therefore, the additional constraint does not affect the optimal solution and ensures a feasible bounded region for the LP.
The initial solution obtained for the reduced subset of constraints over the redefined \ac{MRC} \ac{LP} can be trivial, that is, zero. However, Algorithm~\ref{alg:efficient_mrc_for_large_n_m} eventually obtains a non-trivial solution by iteratively selecting the constraints and features. 

The proposed algorithm for \acp{MRC} provides efficient learning for multiple scenarios of large-scale learning. The theoretical analysis shows that the algorithm can deal with accuracy vs efficiency trade-offs by selecting appropriate values for the hyper-parameters $\epsilon_1$ and $\epsilon_2$. In the next section, we further assess the efficiency achieved through numerical experiments using multiple large-scale datasets.


\begin{table}[ht]
 \captionsetup{labelfont={it}, labelsep=period, font=small, skip=5pt}
                         \caption{Datasets.}
    \vskip -0.15in
     \label{tb:datasets}
\setstretch{1.2}
\begin{center}
\scalebox{0.85}{\begin{tabular}{|c|c|c|c|}
\hline
Dataset & \multicolumn{1}{c|}{Samples ($n$)} & \multicolumn{1}{c|}{Features ($d$)} & \multicolumn{1}{c|}{Classes ($|\set{Y}|$)} \\ \hline
pulsar & 17898 & 8 & 2 \\
house16 & 22784 & 16 & 2 \\
cats vs dogs & 23262  & 512 & 2 \\
yearbook & 37921 & 512 & 2  \\
rcv1 & 20242 & 47236 & 2\\
real\textunderscore sim & 72309 & 20958 & 2\\
news20 & 19996 & 1355191 & 2\\
satellite & 6435 & 36 & 6\\
dry\textunderscore bean & 13611 & 16 & 7 \\
optdigits & 5620 & 64 & 10 \\
mnist & 70000 & 512 & 10 \\
fashion\textunderscore mnist & 70000 & 512 & 10 \\
cifar10 & 60000 & 512 & 10 \\
cifar100 & 60000 & 512 & 100 \\
 \hline
\end{tabular}}
\end{center}
\vskip -0.2cm
\end{table}
\begin{figure*}[ht]
    \centering
    \begin{subfigure}[b]{0.31\textwidth}
        \centering
        \psfrag{Time in secs}[c][t][0.7]{Time in secs}
        \psfrag{upper bound difference}[c][t][0.7]{$\up{R}^k - \up{R}^{*}$}
        \psfrag{cg}[l][l][0.7]{MRC-CCG}
        \psfrag{gurobi}[l][l][0.7]{MRC-LP}
        \psfrag{123456789123456789123456}[l][l][0.7]{MRC-SUB}
        \psfrag{0}[r][r][0.5]{0}
        \psfrag{a}[r][r][0.7]{$10^{-2}$}
        \psfrag{b}[r][r][0.7]{$10^{-1}$}
        \psfrag{c}[r][r][0.7]{$10^{0}$}
        \psfrag{d}[c][c][0.7]{$10^{3}$}
        \psfrag{50}[r][r][0.7]{50}
        \psfrag{100}[r][r][0.7]{100}
        \psfrag{150}[r][r][0.7]{150}
        \psfrag{0.05}[r][r][0.5]{}
        \psfrag{0.15}[r][r][0.5]{}
        \psfrag{0.25}[r][r][0.5]{}
        \includegraphics[width=\textwidth]{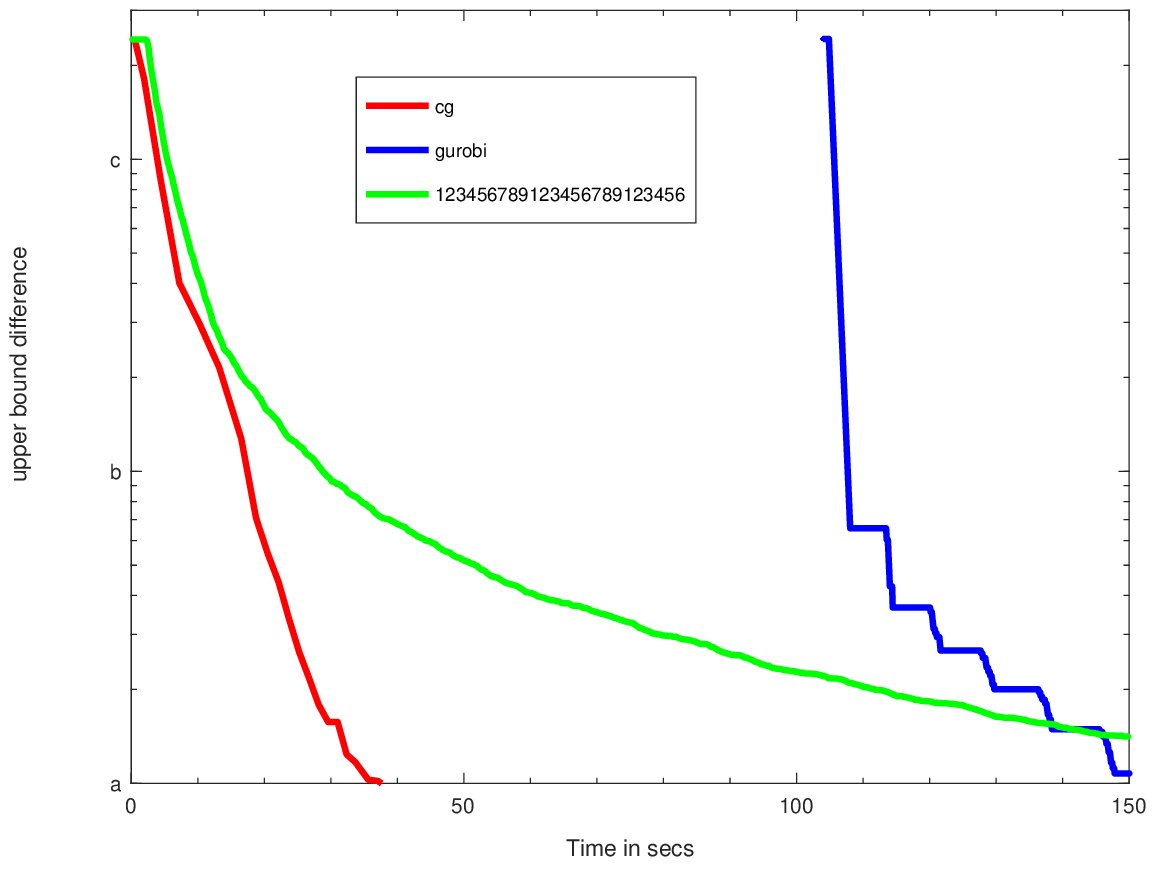}
        \captionsetup{labelfont={it}, font=small, skip=2pt}
        \caption{cats vs dogs dataset}
        \label{fig:acc_vs_time_cats}
    \end{subfigure}
    \begin{subfigure}[b]{0.31\textwidth}
        \centering
        \psfrag{Time in secs}[c][t][0.7]{Time in secs}
        \psfrag{0}[r][r][0.7]{}
        \psfrag{a}[c][c][0.7]{$10^{0}$}
        \psfrag{b}[c][c][0.7]{$10^{2}$}
        \psfrag{c}[c][c][0.7]{$10^{3}$}
        \psfrag{0.05}[r][r][0.7]{0.05}
        \psfrag{0.2}[r][r][0.7]{}
        \psfrag{0.15}[r][r][0.7]{0.15}
        \psfrag{0.4}[r][r][0.7]{}
        \psfrag{0.25}[r][r][0.7]{0.25}
        \includegraphics[width=\textwidth]{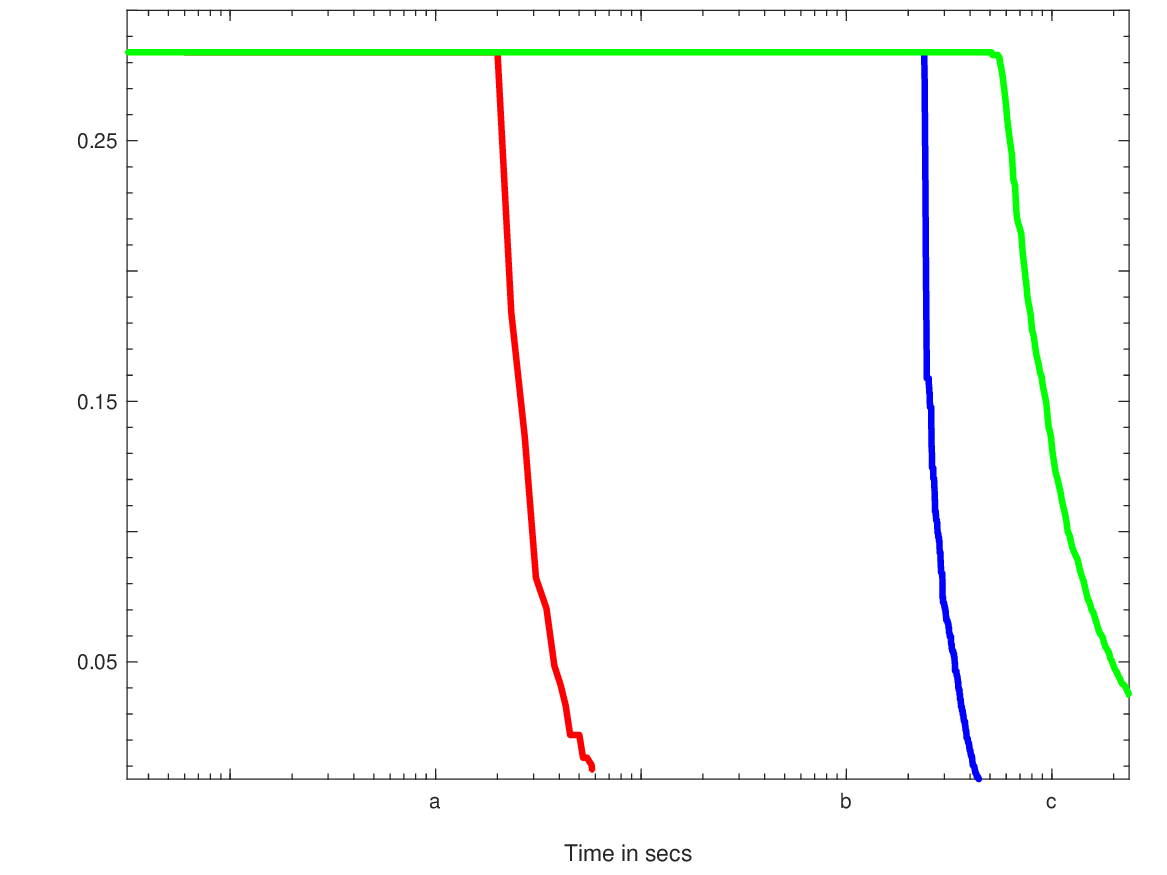}
        \captionsetup{labelfont={it}, font=small, skip=2pt}
        \caption{satellite dataset}
        \label{fig:acc_vs_time_satellite}
    \end{subfigure}
    \begin{subfigure}[b]{0.31\textwidth}
        \centering
        \psfrag{Time in secs}[c][t][0.7]{Time in secs}
        \psfrag{0}[r][r][0.7]{$0$}
        \psfrag{aa}[r][r][0.7]{$10^{-1}$}
        \psfrag{bb}[r][r][0.7]{$10^{0}$}
        \psfrag{cc}[r][r][0.7]{$10^{1}$}
        \psfrag{12}[r][r][0.7]{$12$}
        \psfrag{8}[r][r][0.7]{$8$}
        \psfrag{4}[r][r][0.7]{$4$}
        \psfrag{a}[c][c][0.7]{$5000$}
        \psfrag{b}[c][c][0.7]{$10000$}
        \psfrag{c}[c][c][0.7]{$20000$}
        \psfrag{e}[c][c][0.7]{}
        \psfrag{0.1}[r][r][0.7]{0.1}
        \psfrag{0.05}[r][r][0.5]{}
        \psfrag{0.2}[r][r][0.7]{0.2}
        \psfrag{0.15}[r][r][0.5]{}
        \psfrag{0.3}[r][r][0.7]{0.3}
        \psfrag{0.25}[r][r][0.5]{}
        \psfrag{1200}[][][0.7]{}
        \includegraphics[width=\textwidth]{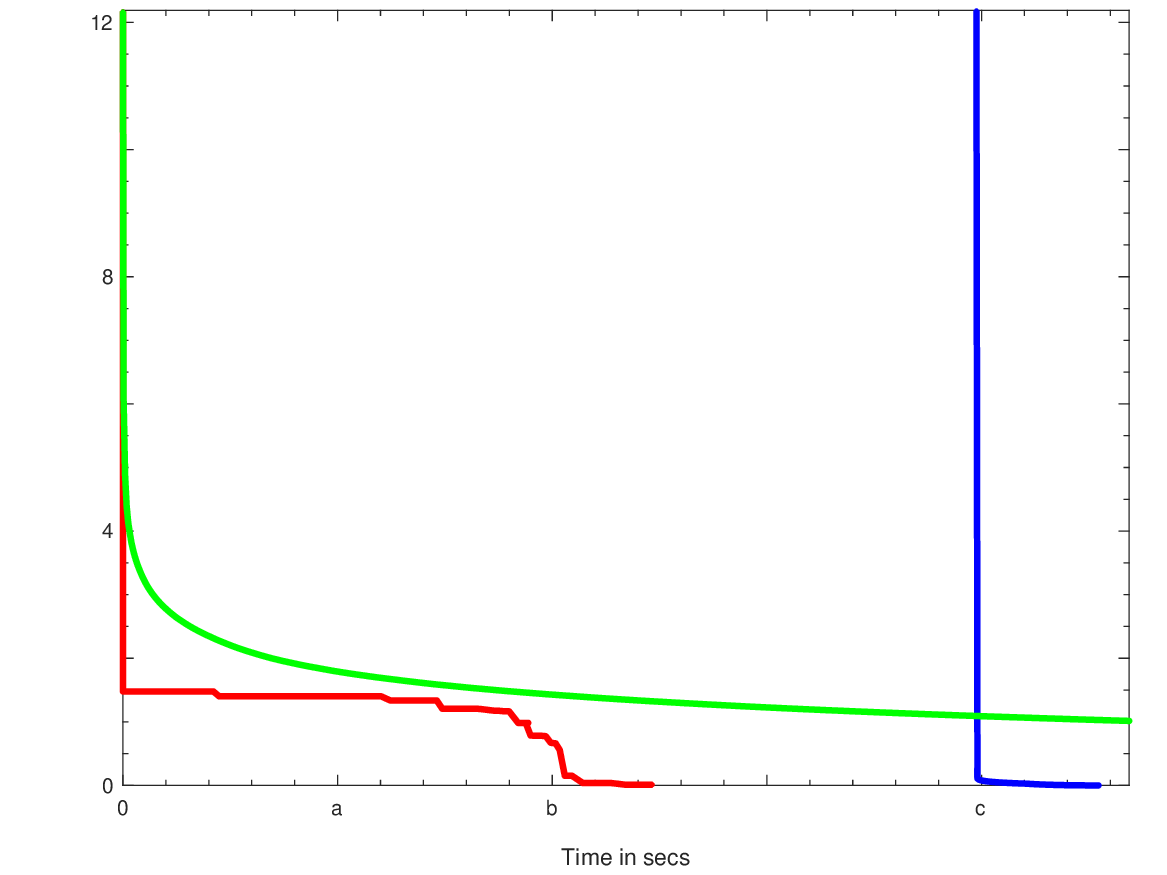}
        \captionsetup{labelfont={it}, font=small, skip=2pt}
        \caption{rcv1 dataset}
        \label{fig:acc_vs_time_rcv1}
    \end{subfigure}
    \captionsetup{labelfont={it}, labelsep=period, font=small, skip=2pt}
    \caption{Convergence of the worst-case error probability $\up{R}^k$ over time. The figures correspond to different scenarios of large-scale learning and demonstrate that \mbox{MRC-CCG} achieves a fast convergence in comparison with state-of-the-art learning methods for \acp{MRC}.}
    \label{fig:acc_vs_time}
\end{figure*}

\vspace{-0.1cm}
\section{Results}
\vspace{-0.1cm}
In this section, we present experimental results that show comparison with state-of-the-art methods in terms of precision in the solution and training time.
All the results are obtained using the hyper-parameters $\epsilon_1=$1e-2, \mbox{$\epsilon_2=$1e-5}, and $n_\text{max}=m_\text{max}=400$ for the proposed algorithm (see Appendix~\ref{subsec:hyper_tuning} for results with other hyper-parameter settings). The experiments are performed in Python 3.9 with a memory limit of 180GB and a time limit of 2e+5 secs. In addition, we use Gurobi optimizer 9.5.2 in our experiments for solving \acp{LP}. 

The proposed algorithm (\mbox{MRC-CCG}) is compared with the \ac{LP} formulation of \acp{MRC} \eqref{eq:mrc_linear_primal} using all constraints and variables (\mbox{MRC-LP}), and the accelerated subgradient method (\mbox{MRC-SUB}) presented in \cite{MazRomGrun:23} which is implemented in the python library MRCpy \citep{BonMazPer:24}. In addition, the proposed algorithm is compared with constraint generation methods for L1-\acp{SVM} (\mbox{SVM-CCG}) presented in \cite{DedAntEtal:22} for binary classification and the L1-\acp{SVM} for multi-class classification (\mbox{SVM-MULTI}) presented in \cite{wang2006l_1}. We set the regularization parameter for all the methods to 0.01.

The experimental results are obtained using 7 binary and 7 multi-class datasets summarized in Table~\ref{tb:datasets}. We use the ResNet18 \citep{he2016deep} features for the image datasets `cats vs dogs', `yearbook', `mnist', `fashion\textunderscore mnist', `cifar10', and `cifar100', and augment the features of the tabular datasets `pulsar', `house16', `satellite', `dry\textunderscore bean', and `optdigits' with 400 random Fourier features. The proposed algorithm (MRC-CCG) is implemented in the python library MRCpy \citep{BonMazPer:24}.

\begin{figure*}
    \begin{center}
    \begin{subfigure}[b]{0.31\textwidth}
        \centering
        \psfrag{Fraction of samples}[c][t][0.8]{Fraction of samples}
        \psfrag{Training time}[c][t][0.8]{Time in secs}
        \psfrag{mrc-cg}[l][l][0.7]{MRC-CCG}
        \psfrag{123456789123456789123456}[l][l][0.7]{MRC-SUB}
        \psfrag{gurobi}[l][l][0.7]{MRC-LP}
        \psfrag{0}[r][r][0.5]{}
        \psfrag{0.1}[c][c][0.7]{0.1}
        \psfrag{0.2}[][][0.7]{}
        \psfrag{0.3}[c][c][0.7]{0.3}
        \psfrag{0.4}[][][0.7]{}
        \psfrag{0.5}[c][c][0.7]{0.5}
        \psfrag{0.6}[][][0.7]{}
        \psfrag{0.7}[c][c][0.7]{0.7}
        \psfrag{a}[r][r][0.7]{$10^1$}
        \psfrag{b}[r][r][0.7]{$10^2$}
        \psfrag{c}[r][r][0.7]{$10^3$}
        \includegraphics[width=\textwidth]{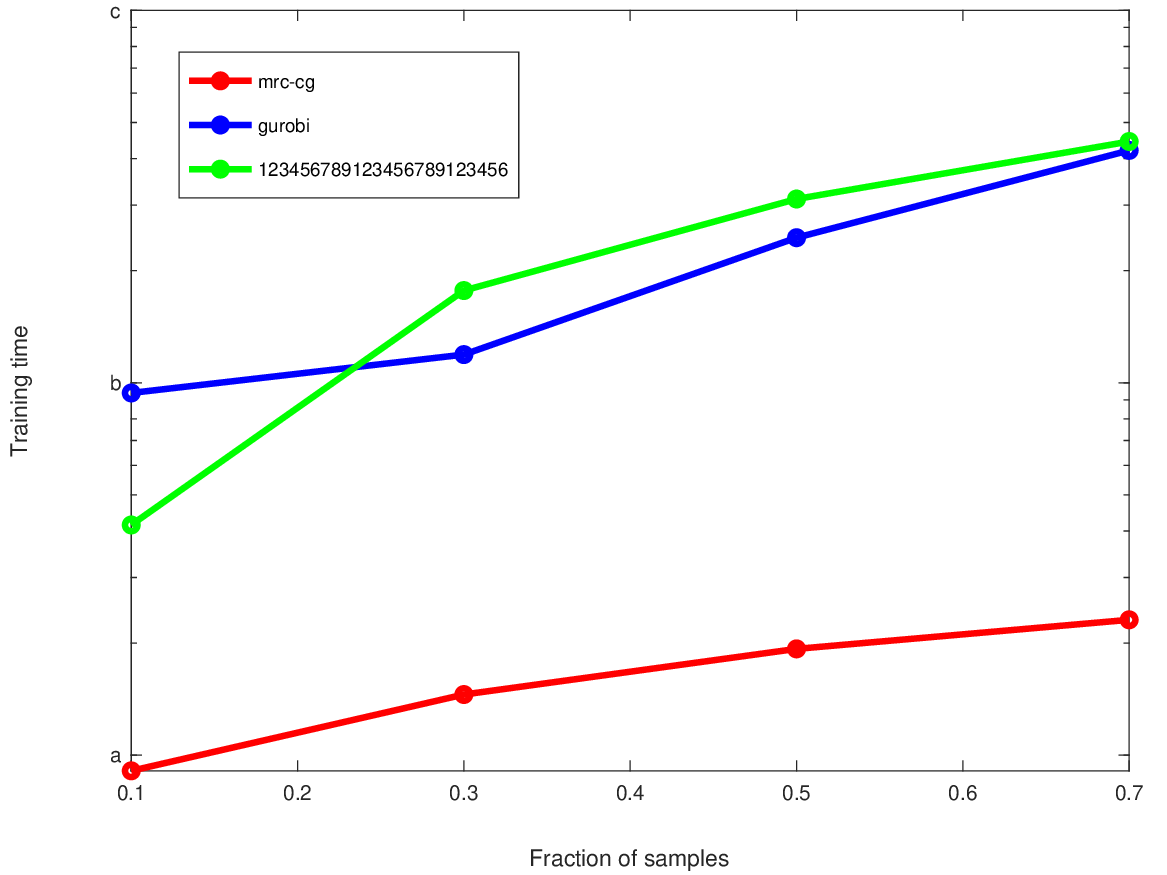}
        \captionsetup{labelfont={it}, font=small, skip=2pt}
        \caption{cats vs dogs dataset}
        \label{fig:scalability_cats}
    \end{subfigure}
    \begin{subfigure}[b]{0.31\textwidth}
        \centering
        \psfrag{Fraction of samples}[c][t][0.8]{Number of classes}
        \psfrag{Training time}[c][t][0.8]{}
        \psfrag{3}[c][c][0.7]{3}
        \psfrag{5}[c][c][0.5]{}
        \psfrag{7}[c][c][0.7]{7}
        \psfrag{9}[c][c][0.7]{}
        \psfrag{11}[c][c][0.7]{11}
        \psfrag{13}[c][c][0.5]{}
        \psfrag{15}[c][c][0.7]{15}
        \psfrag{17}[c][c][0.5]{}
        \psfrag{19}[c][c][0.7]{19}
        \psfrag{20}[c][c][0.5]{}
        \psfrag{a}[r][r][0.7]{$10^2$}
        \psfrag{b}[r][r][0.7]{$10^3$}
        \psfrag{c}[r][r][0.7]{$10^4$}

        \includegraphics[width=\textwidth]{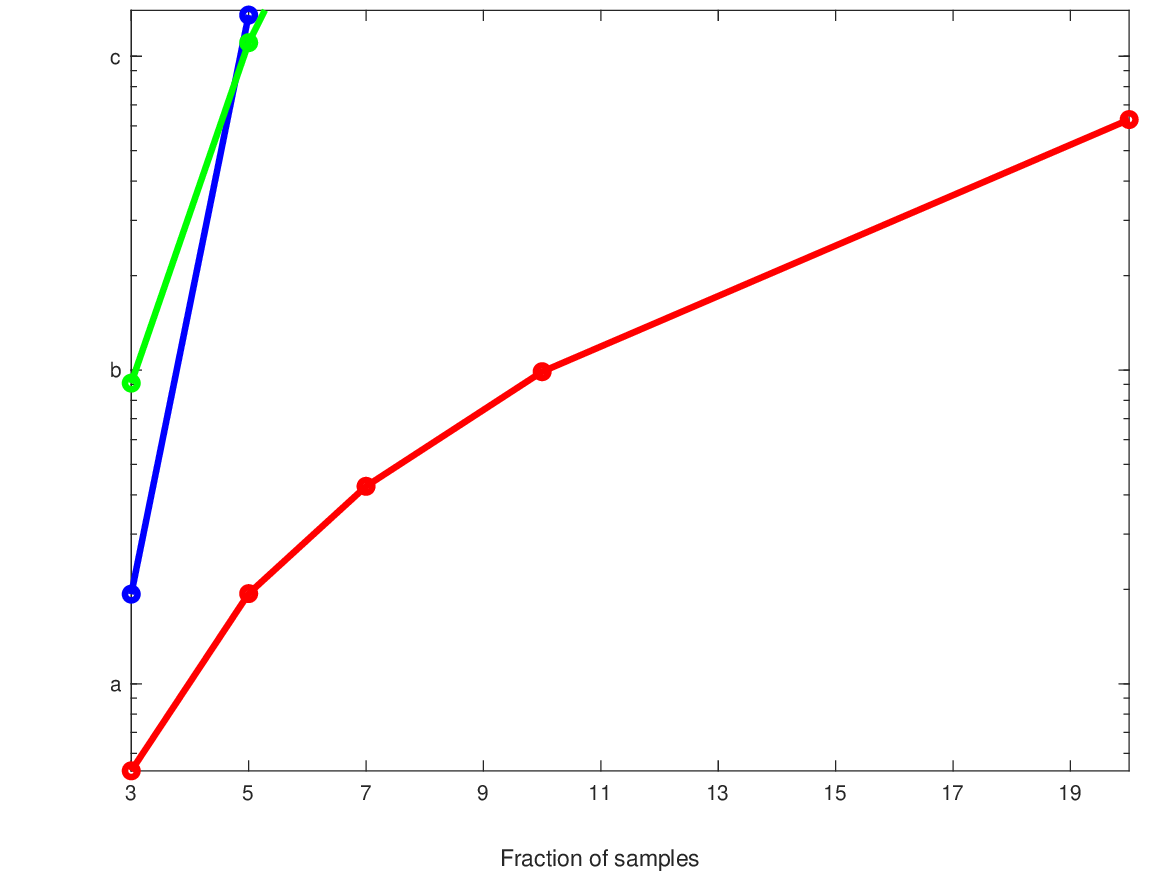}
        \captionsetup{labelfont={it}, font=small, skip=2pt}
        \caption{cifar100 dataset}
        \label{fig:scalability_cifar100}
    \end{subfigure}
    \begin{subfigure}[b]{0.31\textwidth}
        \centering
        \psfrag{Fraction of samples}[c][t][0.8]{Fraction of samples and features}
        \psfrag{Training time}[c][t][0.8]{}
        \psfrag{0}[r][r][0.5]{}
        \psfrag{0.1}[c][c][0.7]{0.1}
        \psfrag{0.2}[][][0.7]{}
        \psfrag{0.3}[c][c][0.7]{0.3}
        \psfrag{0.4}[][][0.7]{}
        \psfrag{0.5}[c][c][0.7]{0.5}
        \psfrag{0.6}[][][0.7]{}
        \psfrag{0.7}[c][c][0.7]{0.7}
        \psfrag{2000}[r][r][0.7]{}
        \psfrag{4000}[r][r][0.7]{}
        \psfrag{8000}[r][r][0.7]{}
        \psfrag{10000}[r][r][0.7]{}
        \psfrag{6000}[r][r][0.7]{$6000$}
        \psfrag{12000}[r][r][0.7]{$12000$}
        \psfrag{14000}[r][r][0.7]{}
        \includegraphics[width=\textwidth]{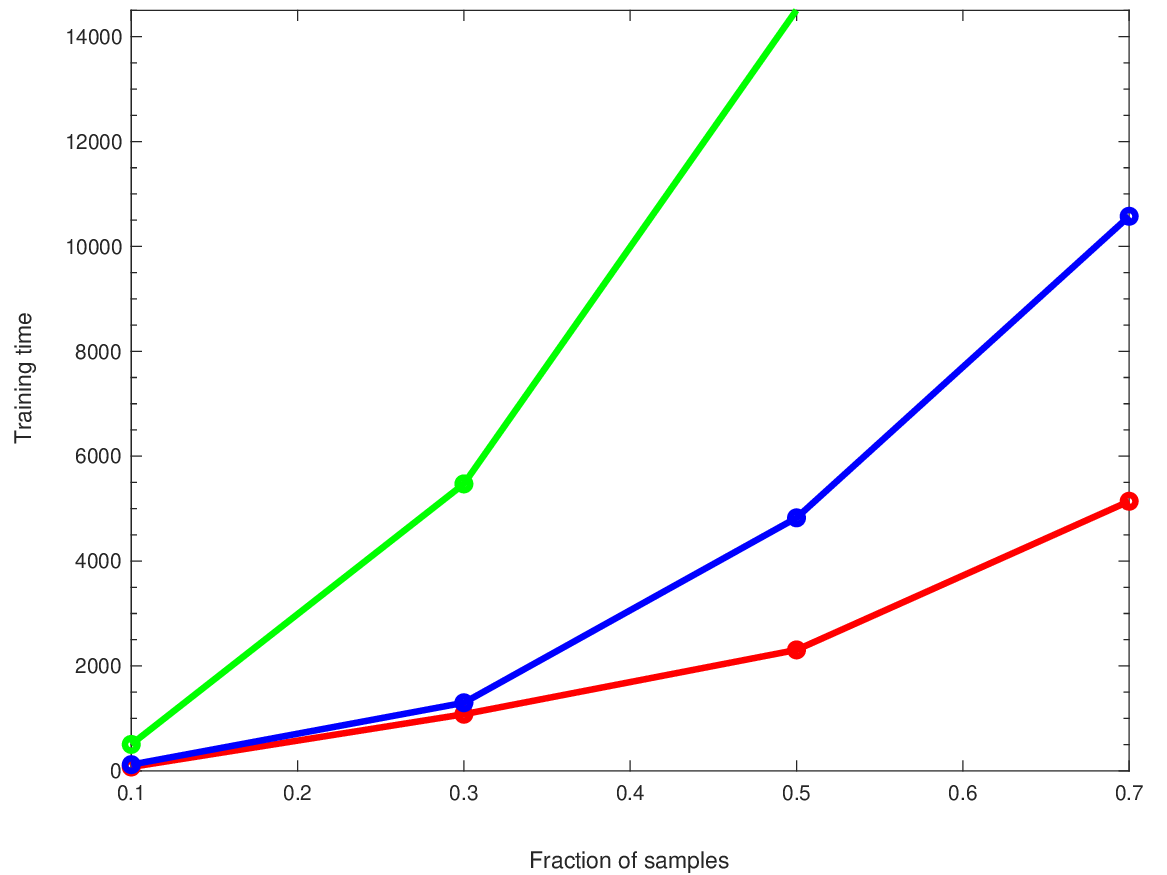}
        \captionsetup{labelfont={it}, font=small, skip=2pt}
        \caption{rcv1 dataset}
        \label{fig:scalability_rcv1}
    \end{subfigure}
    \captionsetup{labelfont={it}, labelsep=period, font=small, skip=2pt}
    \caption{Illustration of scalability for multiple scenarios of large-scale learning. The figures demonstrate that \mbox{MRC-CCG} achieves an improved scalability in comparison with state-of-the-art learning methods for \acp{MRC}. }
    \label{fig:scalability}
    \end{center}
    \vskip -0.3in
\end{figure*}

\vspace{-0.1cm}
\subsection{Fast convergence and improved scalability}
\vspace{-0.1cm}
Figure~\ref{fig:acc_vs_time} illustrates the fast convergence to $\up{R}^{*}$ achieved by the proposed iterative algorithm \mbox{MRC-CCG} in comparison with the other algorithms \mbox{MRC-SUB} and \mbox{MRC-LP}. The worst-case error $\up{R}^{k}$ is obtained along the iterations \mbox{$k=1,2,\ldots$} of each algorithm and the optimal worst-case error $\up{R}^{*}$ corresponds to the optimal worst-case error obtained by the \mbox{MRC-LP} that solves \eqref{eq:mrc_linear_primal} using all the constraints and variables.
Figure~\ref{fig:acc_vs_time_cats} and Figure~\ref{fig:acc_vs_time_satellite} correspond with a binary and multi-class classification problems with a large number of samples while Figure~\ref{fig:acc_vs_time_rcv1} corresponds to a binary classification problem with a large number of samples and features. All the times are averaged over 10 repetitions using 70\% of the data.
The figures show that \mbox{MRC-CCG} converges to $\up{R}^{*}$ significantly faster than the existing methods \mbox{MRC-SUB} and \mbox{MRC-LP} in all scenarios. Such improvement is even more clear for the multi-class scenario due to the complexity improvement from exponential to quasi-linear in the number of classes. Moreover, we observe that the difference in the worst-case error $\up{R}^{k} - \up{R}^{*}$ monotonically decreases with increasing time in case of large number of samples (Figure~\ref{fig:acc_vs_time_cats} and~\ref{fig:acc_vs_time_satellite}) due to the properties shown in Theorem~\ref{th:convergence_n_greater_than_m}.

Figure~\ref{fig:scalability} illustrates the improved scalability of \mbox{MRC-CCG} in comparison with the existing methods \mbox{MRC-SUB} and \mbox{MRC-LP}. Specifically, the figure shows the training times obtained for increasing the number of samples, the number of classes, and the number of both samples and features. All the training times are averaged over 10 random partitions of data. The figures show that \mbox{MRC-CCG} achieves a significantly more efficient learning of \acp{MRC} with increasing data and classes. In particular, Figure~\ref{fig:scalability_cifar100} shows that the algorithm achieves significant improvement in training times for multi-class setting and scales to a sizeable number of classes.

\vspace{-0.1cm}
\subsection{Training times and precision}
\vspace{-0.1cm}
\begin{table*}[ht]
     \captionsetup{labelfont={it}, labelsep=period, font=small, width=.85\textwidth, skip=0pt}
      \caption{Performance assessment in terms of training time (in secs) and average error (AE) for binary classification. }
      \vskip -0.2cm
      \label{tb:time_accuracy_binary}
\setstretch{1.2}
\begin{center}

\scalebox{0.8}{\begin{tabular}{|c|c|c|c|c|c|c|c|c|c|}
\hline
\multicolumn{1}{|c|}{\multirow{3}{*}{Dataset}} & \multicolumn{3}{c|}{\acp{SVM}} & \multicolumn{5}{c|}{\acp{MRC}} \\
 \cline{2-9}
 & \multicolumn{1}{c|}{LP} & \multicolumn{2}{c|}{CCG} & \multicolumn{1}{c|}{LP} & \multicolumn{2}{c|}{SUB} & \multicolumn{2}{c|}{CCG} \\
 \cline{2-9}
 & Time & Time & AE & Time & Time & AE & Time & AE \\
\hline
pulsar                                      & 1.1e+2 $\pm$ 1e+0 & 1.6e+2 $\pm$ 1e+1 & 5e-11  & 7.0e+1 $\pm$ 5e+0 &  2.5e+2 $\pm$ 7e+1 & 4e-3  & 7.2e+0 $\pm$ 1e+0 & 8e-5 \\
house16                                     & 1.3e+2 $\pm$ 1e+0 & 1.7e+2 $\pm$ 3e+0 & 1e-11  & 9.4e+1 $\pm$ 1e+0 &  3.2e+2 $\pm$ 8e+1 & 3e-3  & 7.1e+0 $\pm$ 4e+0 & 9e-5 \\
cats vs dogs                                & 1.7e+2 $\pm$ 8e+0 & 1.0e+2 $\pm$ 9e+0 & 2e-11  & 2.0e+2 $\pm$ 2e+1 &  4.0e+2 $\pm$ 1e+2 & 6e-4  & 2.9e+1 $\pm$ 5e+0 & 1e-5 \\
yearbook                                    & 2.8e+2 $\pm$ 1e+1 & 3.0e+2 $\pm$ 3e+1 & 1e-11  & 3.3e+2 $\pm$ 2e+1 &  6.2e+2 $\pm$ 1e+2 & 3e-4  & 3.8e+1 $\pm$ 5e+0 & 5e-6 \\
rcv1                                        & 5.6e+3 $\pm$ 2e+2 & 2.5e+2 $\pm$ 2e+1 & 1e-4   & 2.1e+4 $\pm$ 1e+3 &  3.6e+4 $\pm$ 8e+3 & 6e-2  & 8.1e+3 $\pm$ 2e+3 & 2e-3 \\
real\textunderscore sim                     & 9.4e+3 $\pm$ 2e+2 & 6.6e+3 $\pm$ 2e+3 & 4e-1   & 5.7e+4 $\pm$ 6e+3 &  6.2e+4 $\pm$ 2e+4 & 5e-2  & 1.9e+4 $\pm$ 8e+2 & 9e-4 \\
news20                                      & -                 & 2.1e+4 $\pm$ 5e+3 & -      & -                 &  -                 &  -    & 1.4e+5 $\pm$ 8e+3 & -    \\
\hline
\end{tabular}}
\end{center}
\vskip -0.6cm
\end{table*}

\begin{table}[ht]
 \captionsetup{labelfont={it}, labelsep=period, font=small, width=.47\textwidth, skip=0pt, belowskip=0pt}
                         \caption{Training times (in secs) for multi-class classification. }
    \vskip -0.2cm
     \label{tb:time_accuracy_multi}
\setstretch{1.2}
\begin{center}
\scalebox{0.8}{\begin{tabular}{|c|c|c|c|}
\hline
\multicolumn{1}{|c|}{\multirow{1}{*}{Dataset}} & \multicolumn{1}{c|}{SVM-MULTI} & \multicolumn{1}{c|}{MRC-SUB} & \multicolumn{1}{c|}{MRC-CCG} \\
 \cline{1-4}
satellite                                & 3.5e+2 $\pm$ 6e+0 &  4.3e+2 $\pm$ 1e+2 & 7.4e+1 $\pm$ 5e+0 \\
dry\textunderscore bean                  & 1.6e+2 $\pm$ 7e+0 &  2.5e+3 $\pm$ 1e+2 & 4.2e+1 $\pm$ 5e+0 \\
optidigits                               & 4.9e+2 $\pm$ 1e+1 &  5.8e+3 $\pm$ 4e+2 & 3.4e+2 $\pm$ 2e+1 \\
mnist                                    & 9.2e+4 $\pm$ 5e+4 &  1.8e+5 $\pm$ 4e+4 & 2.5e+3 $\pm$ 3e+2 \\
fashion\textunderscore mnist             & -                 &  1.9e+5 $\pm$ 2e+4 & 2.9e+3 $\pm$ 1e+2 \\
cifar10                                  & -                 &  1.2e+5 $\pm$ 2e+4 & 4.2e+3 $\pm$ 2e+2 \\
cifar100                                 & -                 &  -                 & 9.7e+4 $\pm$ 3e+3 \\
\hline
\end{tabular}}
\end{center}
\vskip -0.2cm
\end{table}

Table~\ref{tb:time_accuracy_binary} quantifies the improvement in training time and the precision in the objective value (in terms of average error (AE)) obtained by \mbox{MRC-CCG} with respect to \mbox{MRC-LP} and \mbox{MRC-SUB} using multiple binary classification datasets. All the results are averaged over 10 random repetitions using 80\% of the data and the AE in the objective value is computed as the difference with respect to the objective value obtained by solving the \ac{LP} \eqref{eq:mrc_linear_primal} using all the constraints and variables. In addition, the table presents results for the constraint generation methods for L1-\acp{SVM} (\mbox{SVM-CCG}) and the \ac{LP} formulation of L1-\acp{SVM} (SVM-LP). Notice that the large AE corresponding to \acp{SVM} represents the differences in large objective.

Table~\ref{tb:time_accuracy_multi} quantifies the improvement in training time by \mbox{MRC-CCG} with respect to \mbox{MRC-SUB}, and the L1-\acp{SVM} for multi-class classification (SVM-MULTI) since the efficient learning method (SVM-CCG) for L1-SVM is proposed for binary classification problems \citep{DedAntEtal:22}. Note that the empty cells in the table for \mbox{SVM-LP}, \mbox{SVM-MULTI}, \mbox{MRC-LP}, and \mbox{MRC-SUB} correspond to the cases where the method could not obtain a solution due to their requirements for computational resources. 

Table~\ref{tb:time_accuracy_binary} shows that the proposed algorithm \mbox{MRC-CCG} obtains accurate solutions with low AE (see Appendix~\ref{subsec:more_results} for the classification errors on each dataset) and achieves around a 10x speedup for datasets with a large number of samples (`pulsar', `house16', `cats vs dogs', and `yearbook'), and around a 2x speedup for datasets with a large number of samples and features (`rcv1', `real\textunderscore sim', and `news20'). In addition, Table~\ref{tb:time_accuracy_multi} shows that the algorithm enables the efficient learning of \acp{MRC} for multi-class settings and can achieve around a 100x speedup with increasing number of classes.

The experimental results show that the proposed algorithm \mbox{MRC-CCG} achieves faster convergence than existing methods with differences in the worst-case error probability lower than 1e-3. The algorithm enables to significantly improve the learning efficiency with large-scale data, especially for cases with large number of samples. In addition, the algorithm provides improved scalability upto a sizeable number of classes.

\vspace{-0.2cm}
\section{Conclusion}
\vspace{-0.2cm}
In this paper, we presented an algorithmic framework based on the combination of constraint and column generation for the efficient learning of \acfp{MRC} in general large-scale scenarios. 
The iterative algorithm utilizes a greedy constraint selection approach at each iteration that results in a complexity that scales quasi-linearly with the number of classes.
We present theoretical results that describe the convergence of the proposed algorithm for multiple scenarios of large-scale data. In particular, the results bound the difference in the worst-case error probability at each iteration.
The numerical results assess the efficiency of the proposed algorithm in comparison with state-of-the-art methods using multiple binary and multi-class datasets. 
The results show that the proposed algorithm provides a significant efficiency increase for large-scale data. Moreover, the algorithm achieves a significantly reduced computational complexity for multi-class classification tasks (upto 100x speedup) and enables scalability to a sizeable number of classes.

\section*{Acknowledgement}
\vspace{-0.2cm}
Funding in direct support of this work has been provided by
projects PID2022-137063NBI00, PLEC2024-011247, and CEX2021-001142-S funded
by MCIN/AEI/10.13039/501100011033 and the European Union
“NextGenerationEU”/PRTR, and program BERC-2022-2025 funded by the
Basque Government. Kartheek
Bondugula also holds a predoctoral grant (EJ-GV 2022) from the Basque Government.

\appendices
\vspace{-0.4cm}
\section{Proof of theorem 1}
\vspace{-0.2cm}
\label{subsec:proof_th1}
\begin{proof}
    At each iteration, the algorithm adds and removes constraints to the primal \ac{LP} \eqref{eq:mrc_linear_primal}, that is, variables to the dual \ac{LP} \eqref{eq:mrc_linear_dual}. The inequality \eqref{ineq:increasing_worst_case_risk} follows by noting that the warm-start to the dual at iteration $k + 1$ is a feasible solution with an objective value that equals the worst-case error probability $\up{R}^k$ at iteration $k$. In particular, the warm-start is feasible at $k + 1$ and has value $\up{R}^k$ because it is obtained from the solution at iteration $k$ by removing and adding features corresponding to zero coefficients.
    In the following, we prove the inequality \eqref{ineq:convergence_worst_case_risk_large_n}.

    The second inequality in \eqref{ineq:convergence_worst_case_risk_large_n} follows by noting that $\up{R}^{*}$ is the minimum value corresponding with all the constraints and will be greater than or equal to any optimal value corresponding with a subset of constraints. The first inequality in \eqref{ineq:convergence_worst_case_risk_large_n} is obtained as follows.
    
    
    Let $\B{\mu}_1^{k_0}, \B{\mu}_2^{k_0} \in \mathbb{R}^{m}$, $\nu^{k_0} \in \mathbb{R}$ be the primal solution at iteration $k_0$ for a subset of primal constraints $\set{I}$, then we have
    \begin{align}
        \label{eq:relaxed_constraints_primal}
        \V{F}(\B{\mu}_1^{k_0} - \B{\mu}_2^{k_0}) - \nu^{k_0}\B{1} \preceq \V{b} + \hat{\epsilon}_1\B{1}
    \end{align}
    since $\hat{\epsilon}_1$ is the largest violation in the constraints of the primal. Now, consider the following dual problem corresponding to the primal \eqref{eq:mrc_linear_primal} with constraints as in \eqref{eq:relaxed_constraints_primal}.
    \begin{align}
    \label{eq:mrc_linear_dual_relaxed}
    \def\arraystretch{1}
    \begin{array}{cc}
    \underset{\B{\alpha} \in \mathbb{R}^p, \B{\alpha} \succeq 0}{\max} & -(\B{\up{b}} + \hat{\epsilon}_1\B{1})^{\up{T}}\B{\alpha} \\
    \text{s.t.} & \B{\tau} - \B{\lambda} \ \preceq \ \B{\up{F}}^{\up{T}}\B{\alpha} \ \preceq  \ \B{\tau} + \B{\lambda}, \ \B{1}^{\up{T}}\B{\alpha} = 1.
    \end{array}
    \end{align}
    If $\up{R}^{*}$ is the optimal solution of the \ac{MRC} \ac{LP} \eqref{eq:mrc_linear_primal} using all the constraints, then the corresponding dual solution $\B{\alpha}^{*}$ is also a feasible solution for the dual \ac{LP} \eqref{eq:mrc_linear_dual_relaxed}. Therefore, by weak duality (\citep[Theorem 4.3]{DimTsi97}), we have that
    \begin{align}
        \label{eq:weak_duality_eps_1}
        \nonumber -(\B{\up{b}} + \hat{\epsilon}_1\B{1})^{\up{T}}\B{\alpha}^{*} \leq \ & (\B{\lambda} - \B{\tau})^{\up{T}}\B{\mu}^{k_0}_1 \\ & + (\B{\tau} + \B{\lambda})^{\up{T}}\B{\mu}^{k_0}_2 + \nu^{k_0}
 = \up{R}^{k_0}
    \end{align}
    since $\B{\mu}_1^{k_0}, \B{\mu}_2^{k_0}$, and $\nu^{k_0}$ is a feasible solution for the primal \ac{LP} corresponding with relaxed constraints in \eqref{eq:relaxed_constraints_primal}. Therefore, the first inequality in \eqref{ineq:convergence_worst_case_risk_large_n} is obtained because
    \begin{align}
        \label{eq:final_inequality_th1}
        \up{R}^{*} - \hat{\epsilon}_1\B{1}^{\up{T}}\B{\alpha}^{*} = \up{R}^{*} - \hat{\epsilon}_1 \leq \ & \up{R}^{k_0}
    \end{align}
   since $\B{1}^{\up{T}}\B{\alpha}^{*} = 1$. Hence, the first inequality of \eqref{ineq:convergence_worst_case_risk_large_n} holds for any $k \geq k_0$ since $\up{R}^{k} \geq \up{R}^{k_0}$ due to the monotonic increase of the worst-case error probability. 
\end{proof}

\vspace{-0.3cm}
\section{Proof of theorem 2}
\vspace{-0.2cm}
\label{subsec:proof_th2}
\begin{proof}
   The first inequality follows from Theorem~\ref{th:convergence_n_greater_than_m}, based on the maximum primal constraint violation $\hat{\epsilon}_1$ at iteration $k$.

    We now prove the second inequality by applying weak duality \citep{DimTsi97}. Let $\B{\mu}_1^k, \B{\mu}_2^k \in \mathbb{R}^m$, $\nu^k \in \mathbb{R}$ be the primal solution, and $\B{\alpha}^k \in \mathbb{R}^p$ the dual solution for selected constraints $\set{I}$ and features $\set{J}$. Since $\hat{\epsilon}_2$ is the maximum violation in the dual constraints, we have:
    \begin{align}
        \label{eq:relaxed_constraints_dual}
        \B{\tau} - \B{\lambda} - \hat{\epsilon}_2\B{1} \ \preceq \ \B{\up{F}}^{\up{T}}\B{\alpha}^{k} \ \preceq  \ \B{\tau} + \B{\lambda} + \hat{\epsilon}_2\B{1}
    \end{align}
    Now, we consider the corresponding relaxed primal problem. The optimal \ac{MRC} solution $(\B{\mu}_1^*, \B{\mu}_2^*, \nu^*)$ remains feasible under this relaxation. Then, by weak duality:
    \begin{align*}
        -\V{b}^\up{T}\B{\alpha}^k \leq(\B{\lambda} + \hat{\epsilon}_2\B{1} - \B{\tau})^{\up{T}}\B{\mu}^{*}_{1} + (\B{\tau} + \B{\lambda} + \hat{\epsilon}_2\B{1})^{\up{T}}\B{\mu}^{*}_{2} + \nu^*
    \end{align*}
    which implies $
        \up{R}^k \leq \up{R}^* + \hat{\epsilon}_2\|\B{\mu}^*\|$
    since $\B{\mu}^*_1 = (\B{\mu}^*)_+$ and $\B{\mu}^*_2 = (-\B{\mu}^*)_+$.
\end{proof}

\vspace{-0.7cm}
\section{Hyper-parameter tuning}
\vspace{-0.2cm}
\label{subsec:hyper_tuning}
In the following, we present results that analyze the effect of the hyper-parameters $\epsilon_1$, $\epsilon_2$, $n_\text{max}$, and $m_\text{max}$ on the efficiency of Algorithm~\ref{alg:efficient_mrc_for_large_n_m}.

\vspace{-0.2cm}
\subsection{Effect of $\epsilon_1$ and $\epsilon_2$}
\vspace{-0.1cm}
Hyper-parameters $\epsilon_1$ and $\epsilon_2$ determine the threshold for the primal and dual constraints' violations, that is, the proposed algorithm only selects constraints that have a higher violation than the threshold in each iteration. Therefore, $\epsilon_1$ and $\epsilon_2$ influence the total number of constraints and features selected, and thereby, the overall time taken by the algorithm. Moreover, Theorem~\ref{th:convergence_of_cg_cp} shows that $\epsilon_1$ has direct effect on the accuracy of the solution while the effect of $\epsilon_2$ depends on the sparsity in the coefficient $\B{\mu}^*$. Table~\ref{tb:eps_assess} shows the trade-off between the time and accuracy of the algorithm for different values of the hyper-parameters $\epsilon_1$ and $\epsilon_2$. The table presents the average relative error (ARE), the classification error, and the training time obtained by the algorithm using the datasets "rcv1" and "real\textunderscore sim" with large number of samples and features. The results are obtained for the hyper-parameters \mbox{$n_{\up{max}}=m_{\up{max}}=400$}. We observe that smaller values of the hyper-parameter can obtain very accurate solutions while taking more time. On average, a good trade-off between accuracy and time is obtained for $\epsilon_1=$1e-2 and $\epsilon_2=$1e-5.
\begin{table}[h]
\vskip -0.1cm
 \captionsetup{labelfont={it}, labelsep=period, font=small}
                         \caption{Performance assessment for different values of $\epsilon_1$ and $\epsilon_2$.}
    \vskip -0.15in
     \label{tb:eps_assess}
\setstretch{1.1}
\begin{center}
\scalebox{0.8}{\begin{tabular}{|c|c|c|c|c|c|c|c|}
\hline
\multicolumn{1}{|c|}{\multirow{3}{*}{$\epsilon_1$}} & \multicolumn{1}{|c|}{\multirow{3}{*}{$\epsilon_2$}} & \multicolumn{6}{|c|}{\multirow{1}{*}{Dataset}} \\
\cline{3-8}
 & & \multicolumn{3}{c|}{rcv1} & \multicolumn{3}{c|}{real\textunderscore sim} \\
 \cline{3-8}
 & & Time & Error & AE & Time & Error & AE \\
 \hline
1e-2           & 1e-3          & 9.8e+1          & .07           & 8.5e-2          & 2.7e+2          & .24          & 4.7e-2          \\
1e-2           & 1e-4          & 3.6e+3          & .05           & 1.0e-2          & 5.6e+3          & .16          & 1.0e-2          \\
\textbf{1e-2}  & \textbf{1e-5} & \textbf{9.6e+3} & \textbf{.04}  & \textbf{2.0e-3} & \textbf{2.0e+4} & \textbf{.13} & \textbf{9.2e-4} \\
1e-3           & 1e-3          & 9.4e+1          & .07           & 8.5e-2          & 3.1e+2          & .24          & 4.7e-2          \\
1e-3           & 1e-4          & 3.0e+3          & .05           & 2.1e-2          & 6.5e+3          & .16          & 1.1e-2          \\
1e-3           & 1e-5          & 1.2e+4          & .04           & 2.2e-3          & 2.5e+4          & .13          & 9.1e-4          \\
1e-4           & 1e-3          & 1.0e+2          & .07           & 8.4e-2          & 3.6e+2          & .24          & 4.6e-2          \\
1e-4           & 1e-4          & 3.3e+3          & .05           & 2.0e-2          & 8.8e+3          & .16          & 1.1e-2          \\
1e-4           & 1e-5          & 8.8e+3          & .04           & 2.2e-3          & 2.9e+4          & .13          & 9.3e-4          \\
\hline
\end{tabular}}
\end{center}
\end{table}

\vspace{-0.5cm}
\subsection{Effect of $n_\text{max}$ and $m_\text{max}$}
Hyper-parameters $n_\text{max}$ and $m_\text{max}$ determine the maximum number of constraints and features selected by the proposed algorithm in each iteration. Therefore, $n_\text{max}$ and $m_\text{max}$ have an effect on the complexity per iteration and the total number of iterations of the proposed algorithm. Particularly, increasing the value of $n_\text{max}$ or $m_\text{max}$ decreases the number of iterations for convergence at the expense of increasing the complexity per iteration. In Figure~\ref{fig:nmax}, we present results that show the effect of $n_\text{max}$ and $m_\text{max}$ on the overall computational complexity of the algorithm on datasets `real\textunderscore sim' and `rcv1'. We set $\epsilon_1=$1e-2 and $\epsilon_2=$1e-5 to obtain accurate results. In practice, we observe that $300 < n_\text{max} = m_\text{max} < 1500$ obtains a good compromise between the number of iterations required for convergence and the complexity per iteration. All the numerical results in the main paper are obtained for $n_\text{max} = m_\text{max} = 400$.
\vspace{-0.3cm}
\begin{figure}[h]
    \centering
    \psfrag{nmax}[c][t][0.7]{\shortstack{Number of constraints and features selected \\ $n_\text{max}$ and $m_\text{max}$}}
    \psfrag{Training time}[c][t][0.7]{Time in hours}
    \psfrag{1234567891234567}[l][l][0.8]{rcv1}
    \psfrag{realsim}[l][l][0.8]{real\textunderscore sim}
    \psfrag{2}[c][c][0.8]{2}
    \psfrag{0}[c][c][0.8]{}
    \psfrag{a}[r][r][0.8]{$5.5$}
    \psfrag{b}[r][r][0.8]{$16.5$}
    \psfrag{c}[r][r][0.8]{$27.5$}
    \psfrag{500}[r][r][0.8]{500}
    \psfrag{1000}[r][r][0.8]{1000}
    \psfrag{1500}[r][r][0.8]{1500} 
    \includegraphics[width=0.44\textwidth]{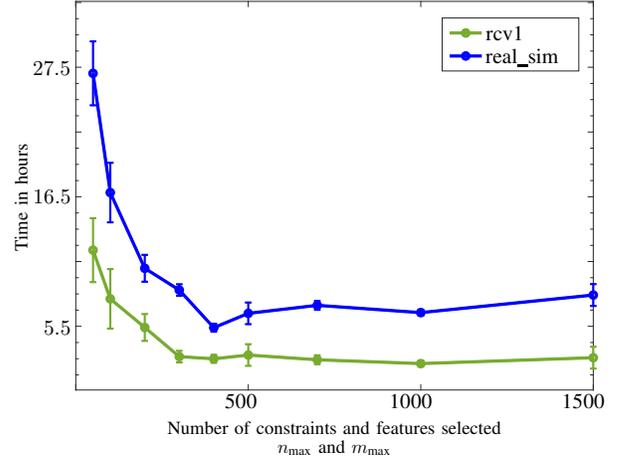}
    \captionsetup{labelfont={it}, labelsep=period, font=small}
    \caption{Effect of hyper-parameters $n_\text{max}$ and $m_\text{max}$ on training time.}
    \label{fig:nmax}
\end{figure}
\vspace{-0.4cm}

\vspace{-0.1cm}
\section{Classification errors along with training times on multiple benchmark datasets}
\label{subsec:more_results}
In this section, we present results to show that the proposed algorithm MRC-CCG enables efficient learning while achieving classification errors similar to the state-of-the-art. Table~\ref{tb:time_error_binary} presents results corresponding to binary classification datasets while Table~\ref{tb:time_error_multi} presents results corresponding to multi-class classification datasets. Moreover, the results show that the proposed algorithm achieves competitive classification errors compared to the methods based on \acp{SVM}. In addition, the results present the worst-case error error probability obtained by the methods for \acp{MRC} that can serve to provide an upper bound on the classification error.
\begin{table*}[h]
 \captionsetup{labelfont={it}, labelsep=period, font=small}
                         \caption{Training time (in secs) and classification error (CE) for binary classification. }
    \vskip -0.15in
     \label{tb:time_error_binary}
\setstretch{1.2}
\begin{center}
\scalebox{0.8}{\begin{tabular}{|c|c|c|c|c|c|c|c|c|c|c|c|c|c|}
\hline
\multicolumn{1}{|c|}{\multirow{3}{*}{Dataset}} & \multicolumn{4}{c|}{\acp{SVM}} & \multicolumn{9}{c|}{\acp{MRC}} \\
 \cline{2-14}
 & \multicolumn{2}{c|}{LP} & \multicolumn{2}{c|}{CCG} & \multicolumn{3}{c|}{LP} & \multicolumn{3}{c|}{SUB} & \multicolumn{3}{c|}{CCG} \\
 \cline{2-14}
& Time & CE & Time & CE & Time & CE & $\up{R}^*$ & Time & CE & $\up{R}^{\text{SUB}}$ & Time & CE & $\up{R}^{\text{CCG}}$\\
\hline
pulsar                     & 1.1e+2 & .02 & 1.6e+2 & .02 & 7.0e+1 &  .04   & .06 & 2.5e+2 & .05 & .06 & 7.2e+0 & .04 & .06 \\
house16                    & 1.3e+2 & .25 & 1.7e+2 & .25 & 9.4e+1 &  .26   & .27 & 3.2e+2 & .27 & .28 & 7.1e+0 & .26 & .27 \\
cats vs dogs               & 1.7e+2 & .02 & 1.0e+2 & .02 & 2.0e+2 &  .01   & .21 & 4.0e+2 & .01 & .21 & 2.9e+1 & .01 & .21 \\
yearbook                   & 2.8e+2 & .07 & 3.0e+2 & .07 & 3.3e+2 &  .11   & .31 & 6.2e+2 & .11 & .31 & 3.8e+1 & .11 & .31 \\
rcv1                       & 5.6e+3 & .03 & 2.5e+2 & .03 & 2.1e+4 &  .04   & .19 & 3.6e+4 & .05 & .25 & 8.1e+3 & .04 & .19 \\
real\textunderscore sim    & 9.4e+3 & .04 & 6.6e+3 & .04 & 5.7e+4 &  .12   & .20 & 6.2e+4 & .25 & .25 & 1.9e+4 & .13 & .20 \\
news20                     & -      & -   & 2.1e+4 & .07 & -      &  -     & -   &   -    & -   &  -  & 1.4e+5 & .04 & .17 \\
\hline
\end{tabular}}
\end{center}
\vskip -0.25in
\end{table*}
\begin{table}[h]
 \captionsetup{labelfont={it}, labelsep=period, font=small}
                         \caption{Training times (in secs) and classification error (CE) for multi-class classification. }
    \vskip -0.15in
     \label{tb:time_error_multi}
\setstretch{1.1}
\begin{center}
\scalebox{0.8}{\begin{tabular}{|c|c|c|c|c|c|c|c|c|}
\hline
\multicolumn{1}{|c|}{\multirow{2}{*}{Dataset}} & \multicolumn{2}{c|}{SVM-MULTI} & \multicolumn{3}{c|}{MRC-SUB} & \multicolumn{3}{c|}{MRC-CCG} \\
 \cline{2-9}
 & Time & CE & Time & CE & $\up{R}^{\text{SUB}}$ & Time & CE & $\up{R}^{\text{CCG}}$ \\
 \hline
satellite                        & 3.5e+2 & .12 & 4.3e+2 & .17 & .32 & 7.4e+1 & .15 & .26 \\
dry\textunderscore bean          & 1.6e+2 & .44 & 2.5e+3 & .57 & .57 & 4.2e+1 & .40 & .49 \\
optdigits                        & 4.9e+2 & .03 & 5.8e+3 & .06 & .37 & 3.4e+2 & .02 & .22 \\
mnist                            & 9.2e+4 & .06 & 1.8e+5 & .05 & .44 & 2.5e+3 & .05 & .41 \\
fashion                          & -      &  -  & 1.9e+5 & .90 & .90 & 2.9e+3 & .17 & .46 \\
cifar\textunderscore 10          & -      &  -  & 1.2e+5 & .90 & .90 & 4.2e+3 & .17 & .57 \\
cifar\textunderscore 100         & -      &  -  & -      & -   & -   & 9.7e+4 & .55 & .66 \\
\hline
\end{tabular}}
\end{center}
\vskip -0.1cm
\end{table}
\vspace{-0.1cm}
\bibliography{main}

\end{document}